\documentclass[journal]{IEEEtran}
\IEEEoverridecommandlockouts
\usepackage{graphics} 
\usepackage{epsfig} 
\usepackage{amsmath}
\usepackage{amssymb}
\usepackage{amsthm} 
\usepackage{color}
\usepackage{bm}
\usepackage{subfig}
\usepackage{cite}
\usepackage{makecell}
\usepackage{multicol}

\usepackage{etoolbox}

\makeatletter
\let\NAT@parse\undefined
\makeatother

\makeatletter
\def\footnoterule{\relax
  \kern-5pt
  \hbox to \columnwidth{\vrule width 0.5\columnwidth height 0.4pt\hfill}
  \kern4.6pt}
\makeatother
    
\DeclareMathAlphabet\mathbfcal{OMS}{cmsy}{b}{n}

\DeclareMathOperator*{\argmin}{arg\,min}

\newtheorem{theorem}{Theorem}
\newtheorem{definition}{Definition}
\newtheorem{lemma}{Lemma}
\newtheorem{assumption}{Assumption}
\newtheorem{proposition}{Proposition}
\AfterEndEnvironment{assumption}{\noindent\ignorespaces}

\title{
On Incremental Structure-from-Motion using Lines
}
\author{Andr\'e Mateus*, Omar Tahri, A. Pedro Aguiar, Pedro U. Lima, and Pedro Miraldo

\thanks{$^\star$ Corresponding author: Andr\'e Mateus}
\thanks{A. Mateus, P. U. Lima, and P. Miraldo are with Institute for Systems and Robotics (ISR/IST), LARSyS, Instituto Superior T\'ecnico, Univ Lisboa, 1049-001 Lisboa, Portugal
        {\tt\small \{andre.mateus, pedro.lima, pedro.miraldo\}@tecnico.ulisboa.pt}.}

\thanks{O. Tahri is with VIBOT ERL CNRS 6000, ImViA, Universit\'e Bourgogne Franche-Comt\'e (UBFC), 71200 Le Creusot, France.
        {\tt\small omar.tahri@u-bourgogne.fr}.}

\thanks{A. P. Aguiar is with the Faculty of Engineering, University of
Porto, 4200-465 Porto, Portugal {\tt\small pedro.aguiar@fe.up.pt}}
}

\begin{document}

\maketitle

\begin{abstract}
Humans tend to build environments with structure, which consists of mainly planar surfaces. From the intersection of planar surfaces arise straight lines. 
Lines have more degrees-of-freedom than points.
Thus, line-based Structure-from-Motion (SfM) provides more information about the environment. 
In this paper, we present solutions for SfM using lines, namely, incremental SfM. 
These approaches consist of designing state observers for a camera's dynamical visual system looking at a 3D line. We start by presenting a model that uses spherical coordinates for representing the line's moment vector.
We show that this parameterization has singularities, and therefore we introduce a more suitable model that considers the line's moment and shortest viewing ray. 
Concerning the observers, we present two different methodologies. 
The first uses a memory-less state-of-the-art framework for dynamic visual systems.  
Since the previous states of the robotic agent are accessible --while performing the 3D mapping of the environment-- the second approach aims at exploiting the use of memory to improve the estimation accuracy and convergence speed.
The two models and the two observers are evaluated in simulation and real data, where mobile and manipulator robots are used.
\end{abstract}


\section{Introduction}
\label{sec:intro}

Recovering 3D structure from 2D images is one of the most studied Computer Vision and Robotics problems.
Obtaining 3D information is crucial for various applications for both communities, from the 3D reconstruction of objects/scenes features, to performing Visual Servoing tasks.
Structure-from-Motion (SfM) is a common approach to solve the 3D reconstruction problem.
The goal is to retrieve 3D structures from a set of 2D images, taken by either a moving camera or multiple cameras observing the same scene \cite{koenderink1991,bartoli2005}.
If more than two images are used, one can use refinement strategies, such as Bundle Adjustment, \cite{triggs1999}, which assumes uncertainty in all the involved parameters and performs optimization.

SfM pipelines start by finding feature correspondences from two views, recovering the relative pose, and then 3D information is recovered, e.g., using triangulation.
Traditional approaches to SfM solve the problem for all available images simultaneously.
However, in practice, images are not available simultaneously, namely if the application focuses on monocular vision.
To prevent solving the entire problem every time a new image is available, some researchers shift the focus to incremental approaches to SfM \cite{schonberger2016}.
Taking advantage of the visual dynamical systems presented by Visual Servoing \cite{chaumette2006}, deterministic nonlinear observers have been used to solve the incremental SfM problem, as shown in \cite{deluca2008,civera2010,spica2013,dixon2003,mateus2018,mateus2019,tahri2015,tahri2017}.
Compared to the traditional monocular SfM approaches, which use only image data and thus can only recover the relative pose up to a scale factor \cite{hartley2003}, these methods exploit the camera velocities (linear and angular) to recover the scene scale.
Since odometry is usually available from autonomous agents, incremental SfM is a particularly appealing strategy for robotics.
Observer based SfM has been applied successfully to Visual Servoing in \cite{spica2017}, to the relative pose estimation of ground robots in \cite{rodrigues2019}, to bearing-based localization in \cite{spica2016}, bearing rigidity maintenance in \cite{schiano2017}, and to drone control in \cite{grabe2015,spica2020}.

The usual type of features for applications involving mapping and localization are points. This is due to, for example, 1) the development of robust point feature detectors like SIFT \cite{lowe1999}, and; 2) the geometrical constraints from epipolar geometry, which is the most exploited strategy for relative pose estimation.
However, these are not always easy to detect and match in 2D images, and their detection and description are still ongoing research problems \cite{detone2018}.
The specific type of features to be used must also be linked to the problem we want to tackle.
Given that most SfM applications focus on human-made environments, we should take advantage of their structure.
These environments consist mostly of planar surfaces, which do not contain much texture, essential to point detection \cite{rodrigues2018}.
However, from the intersection of those planar surfaces, straight lines naturally arise.
With respect to points, lines provide more information about the environment \cite{taylor1995}, they are more common in human-made environments and can be more reliably detected \cite{chandraker2009}.
Furthermore, lines have been successfully used as either an alternative to point features in \cite{smith2006,zhang2011,miraldo2014} or a complement to points in \cite{pumarola2017,zuo2017,gomez2019,mateus2020}.

This paper presents solutions to the incremental SfM problem by exploiting straight line features and the camera velocity.
This problem has been addressed in Chaumette et al. \cite{chaumette1996}, but it requires the time derivative of the lines in the image plane to be known and represent lines by the intersection of two planes, which is an implicit representation.
Furthermore, Spica et al. \cite{spica2014} proposed a method to estimate a cylindrical target radius.
Lines are considered indirectly in this process, but only a subset of the line coordinates can be estimated.
To the best of our knowledge, the current state-of-art methods are our previous works \cite{mateus2018,mateus2019}.
In this work, we exploit the line dynamical system presented in \cite{andreff2002}.
Since that system is not fully observable, we propose two different 3D line representations to evaluate which performs better.
We start by focusing on a minimal line parameterization approach --- 3D lines are four degrees of freedom features --- based on spherical coordinates, also referred to as \emph{Sphere}.
This parameterization was first presented in our previous work \cite{mateus2019}.
However, there are some issues concerning minimal parameterizations for lines.
In this case, those coordinates present singularities in the north and south poles of the sphere.
We present the moment-point representation, which does not present singularities even though it is not minimal, to tackle this issue.
The latter representation --- also referred to as \emph{M-P} --- is shown to be more robust to measurement noise, even though it shows a slightly slower convergence.

Recent incremental SfM approaches resort to state observers and dynamical systems for 3D structure recovering.
For the line-based incremental SfM, two observers are considered.
The first is a memory-less observer (\emph{MLO})\footnote{The name memory-less observer was given to this observer to make a distinction between both observers’ types. The observer is memory-less because it uses only the current measurement and camera velocities to compute the next state estimate.} presented in \cite{deluca2008,spica2013}, and
it works as a traditional state observer.
The estimation is updated based on the error between the current estimate and the system output, i.e., it updates the estimates based on the currently available image.
Given the decrease in prices --of hard drives-- robots are now equipped with larger memory devices, which enable them to store past inputs and outputs of the system.
With this in mind, in the second observer, we exploit a Moving Horizon Observer (\emph{MHO}), which keeps a fixed-sized memory of outputs and inputs.
The state is then recovered by solving an optimization problem that minimizes the measurement error of the outputs in the memory and a prediction term.
This prediction term consists of applying the system dynamics to the previous state estimate.
Simulation and real experiments show that the line representation \emph{M-P} proposed in this work coupled with the \emph{MHO} outperforms the \emph{MLO} in robustness to measurement noise while keeping a similar convergence time.

\subsection{Contributions and Outline of the Paper}

This paper is an extension of our previous works \cite{mateus2018} and \cite{mateus2019}.
The novel contributions of this work concerning \cite{mateus2018} and \cite{mateus2019} are:
\begin{itemize}
    \item A singularity free line parameterization \emph{M-P}, which even though is not minimal, is more robust to measurement noise;
    \item An observer for the \emph{M-P} representation based on the framework presented in \cite{spica2013};
    \item The introduction of memory into the Incremental SfM problem, by exploiting Moving Horizon Observers;
    \item Stability results for the \emph{MHO} framework in \cite{alessandri2008} applied to the \emph{M-P} formulation; and
    \item Simulation and real data results comparing the performance of the \emph{MLO} and of the \emph{MHO} for the \emph{Sphere} and \emph{M-P} parameterizations. 
\end{itemize}
As far as applications of the presented method are concerned, we envisage two types of applications. The first consists of plugging the observer in a control scheme, which requires feature depth estimation, such as Visual Servoing \cite{chaumette2006}, namely its Model Predictive Control variants, as in \cite{roque2020}. The second consists of multi-robot mapping, i.e., multiple robots running a single instance of the proposed observer to allow for distributed multiple line estimation.

The remainder of this work is organized as follows.
Sec.~\ref{sec:relWork} presents the related work.
The \emph{Sphere} and \emph{M-P} 3D line representations, and their respective visual dynamic systems are shown in Sec.~\ref{sec:line_dyn}.
The \emph{MLO} systems is presented in Sec.~\ref{sec:obs_nonlinear}.
The application of the \emph{MHO} to lines is presented in Sec.~\ref{sec:mho_lines}.
Sec.~\ref{sec:expResults} presents the simulation and real robot results.
Finally, the conclusions are presented in Sec.~\ref{sec:conclusion}.
\section{Related Work}
\label{sec:relWork}

In this section, we cover previous works on depth estimation with Deep Learning in \ref{sec:deepDepth}.
Approaches for SfM are presented in Sec.~\ref{sec:sfm} and incremental SfM in Sec.~\ref{sec:isfm}.
We also refer to observer techniques used for state estimation, which use memory, in Sec.~\ref{sec:relMHO}.

\subsection{Deep Learning Depth Prediction}
\label{sec:deepDepth}

Deep Learning has been successfully used for a wide range of applications, such as depth prediction/estimation.
In \cite{liu2015,liu2015b} continuous conditional random fields are exploited to perform depth estimation without geometric priors.
Deep depth estimation is usually posed as a regression problem, which can lead to slow convergence.
In \cite{fu2018} the converge issue is tackled by recasting the problem as an ordinal regression network.
Nevertheless, the majority of methods required supervision, and thus a large amount of annotated data.
To cope with this issue, unsupervised architectures were presented in \cite{garg2016,godard2017,zhan2018}.

The previous methods rely on powerful GPUs for training and inference, but for robotic applications, the use of that hardware leads to an increased price and power consumption.
To deal with the lower computation power of robotic platforms,  \cite{wofk2019} presented an architecture capable of achieving competitive performance while achieving a higher frame-rate in an embedded GPU. Furthermore, it still achieves real-time performance running on a CPU.
An alternative architecture for depth estimation on CPU is presented in \cite{poggi2018}.

Keep in mind that most previous works focus on single image depth estimation, which can only recover depth up to a scale factor unless the true scale of known objects is provided during training.
In this work, we exploit multiple-view and camera velocities to estimate the depth and recover the features' true scale.

\subsection{Structure-from-Motion}
\label{sec:sfm}

There is a vast literature concerning SfM, and it would be impossible to report an extensive state-of-the-art in this paper.
For a comprehensive survey, see \cite{ozyesil2017}.
Full pipelines for large scale SfM framework are available in the literature for city-scale reconstruction, such as \cite{agarwal2011}.
A Bundle Adjustment strategy is used to reduce the time complexity and improve the accuracy of SfM in \cite{wu2013}.
Feature matching between pairs of images is one of the problems in SfM frameworks.
To tackle this issue, in \cite{dellaert2000}, the Expectation-Maximization algorithm is exploited to perform SfM without the need for explicit correspondences.
In \cite{qian2004}, Sequential Monte Carlo methods are used to improve robustness to errors in feature tracking and occlusion.
To deal with outliers in the features correspondences, in \cite{nister2005}, a preemptive RANSAC is proposed.
In \cite{schonberger2016}, a framework for robust and scalable SfM is presented and is evaluated in a wide variety of datasets.
The authors made the method and the data available in a general-purpose pipeline for structure-from-motion denoted as COLMAP.

SfM using lines has been addressed in \cite{taylor1995}, where a cost function based on the distance of edge points on line segments in the image is derived.
A method using infinite lines is presented in \cite{bartoli2005}. Authors use Pl\"ucker coordinates to represent 3D lines. 
In \cite{chandraker2009}, a solution for line reconstruction using a stereo pair is presented.
The cost function is based on the fact that the interpretation planes from two pairs of images, observing a 3D line, intersect in the line.
By exploiting the Manhattan world assumption, \cite{schindler2006} presents a method for the recovery of 3D lines from multiple views, with a wide baseline. 

\subsection{Incremental Structure-from-Motion}
\label{sec:isfm}

Traditional SfM approaches consider a set of images that need not be ordered and have large displacements between them.
Small displacements lead to large errors because of the closeness to degenerative configurations.
However, for applications like visual servo control, a sequential and small baseline set of images is required.
One solution to this problem is to perform SfM online, which consists of the simultaneous localization and mapping (SLAM \cite{cadena2016}) problem.
Even though most approaches focus on point features, e.g. MonoSLAM \cite{davison2007}, RSLAM \cite{mei2011}, ORB-SLAM\cite{mur2015,mur2017,campos2020}, Endres et. al. \cite{endres20133}, and DS-SLAM\cite{yu2018}.
Lines have been used, especially to handle low texture environments \cite{gomez2019}.
Extensions to ORB-SLAM to account for line features are presented in \cite{pumarola2017,zuo2017}.
Line segments have also been successfully used for SLAM while paired with point features in \cite{gomez2019}.

The previous methods allow the reconstruction of a large set of features. However --when using monocular vision-- depth can only be recovered up to a common scale factor.
To achieve metric reconstruction, they require either stereo vision or the 3D structure of a set of features to be known.
To use monocular vision and handle small baselines, some authors took advantage of robotic systems. 
Since an estimate of the velocity is usually available, the dynamical model --of the motion of the observed features in the image concerning the camera velocity-- can be exploited.
These approaches are based on filtering or observer design.
Filtering approaches tend to exploit the use of Kalman Filter and its variants.
In \cite{civera2008}, an Extended Kalman Filter (EKF) is used to reconstruct the 3D environment and camera pose by estimating the inverse depth of points.
The same authors in \cite{civera2010} combine an EKF with RANSAC to solve the multiple point reconstruction problem. 
The reconstruction of a planar target is addressed in \cite{omari2013}, where an Unscented Kalman Filter is used to filter measurements from a visual-inertial system.
EKF has been used for SLAM exploiting lines in \cite{smith2006,zhang2011}.

A limitation of the previous filtering approaches is their reliance on EKF, which requires linearization of the dynamics.
Thus, the estimation is susceptible to poor initial estimates, which may cause the filter to diverge.
To address the linearization issues, nonlinear observers have been used.
Observer-based methods consist of exploiting the visual dynamical systems developed in Visual Servoing \cite{chaumette2006} and design a nonlinear estimation scheme--to retrieve 3D information--given the camera velocities and image measurements.
In \cite{dixon2003}, an observer to estimate the depth of points moving with affine motion is presented.
An estimation scheme to retrieve point depth and camera focal length, given rigid motion of the camera, is presented in \cite{deluca2008}.
Reduced-order observers for range estimation are presented in \cite{morbidi2010,sassano2010,dani2012}.
In \cite{rodrigues2019} the depth estimation and relative pose estimation problems are addressed. The solution exploits an observer for depth estimation, which is then inputted into an EKF that estimates the relative pose.
Filtering and observer-based approaches are compared in \cite{grabe2015}.
With respect to SLAM approaches \cite{gomez2019}, these methods recover depth by combining image measurements and velocity readings, but they estimate a smaller set of features.
Nonetheless, they allow to characterize the convergence of the depth estimate and provide guarantees.

Since the observer-based methods tend to be applied in robotic applications, we can assume that we can control the robot and the camera movement.
Given this assumption, some authors have studied how the movement of the camera affects depth estimation.
In \cite{chaumette1996}, the optimal movements of a camera to estimate 3D information of points, lines, and spheres are presented.
However, the estimation scheme requires the time derivative of the feature motion in the image plane.
More recently, a framework for \emph{Active Structure-from-Motion} was presented in \cite{spica2013}.
This framework has been successfully applied to different visual features.
Namely, points, cylinders and spheres in \cite{spica2014}, planes and image moments in \cite{spica2015,spica2015b}, rotation invariants in \cite{tahri2015,tahri2017}, and lines in \cite{mateus2018,mateus2019}.
A method to couple this framework with a Visual Servoing control loop is presented in \cite{spica2017}.
An alternative framework for active depth estimation of points is presented in \cite{rodrigues2020}, where convergence guarantees are provided, with fewer constraints on the camera motion relative to \cite{spica2013}.

\subsection{Moving Horizon Observers}
\label{sec:relMHO}

Even though the current observer-based approaches allowed to tackle the linearization issue of the EKF based approaches and to derive control laws capable of optimizing state estimation, they do not take into account measurement and modeling errors.
To account for those errors, we propose to exploit Moving Horizon Observers. 
Finite memory observers are commonly referred to as Moving Horizon Observers (\emph{MHO}) or Receding Horizon Observers.
One of these observers' first applications was to deal with the Kalman Filter's limitations, mainly by adding constraints to the state space, thus preventing physically unfeasible state estimates.
In \cite{muske1993}, it was shown that unconstrained \emph{MHO} formulation achieves similar performance to the standard Kalman Filter for linear systems and that introducing constraints improves the robustness of the estimation.
Extensions to distributed linear systems are presented in \cite{farina2010}, and \cite{goodwin2004}.

As far as nonlinear systems are concerned, the cost formulation based on measurement error is presented in \cite{michalska1995} alongside stability and convergence results.
However, neither measurement and modeling noises were considered.
In \cite{rao2003} measurement noise is taken into account in the formulation.
Besides, an arrival cost to the initial state is included in the cost function.
An alternative formulation is presented in \cite{alessandri2008}, where the mean square error of the measurement concerning the state evolution is coupled with a prediction error.
The latter consists of the current estimate's squared error with the prediction made by applying the error-free dynamics to the previous state estimate.
An alternative prediction is presented in \cite{suwantong2014}, where instead of applying the error-free dynamics, a deterministic observer is used to compute the prediction.
However, this requires a larger estimation window to provide similar results.
Since these observers consist of optimizing a cost function over a finite size window, weighting techniques have been used to account for time instants, where the data is non-exciting  in \cite{sui2011}, and for delayed and lost measurements in  \cite{johansen2013}.
\section{Visual Dynamic Model of a Line}
\label{sec:line_dyn}

In this section, the definition of 3D lines using Pl\"ucker coordinates is presented in Sec.~\ref{sec:pluckCoord}.
The minimal model based on spherical coordinates, which we proposed in \cite{mateus2019}, is presented in Sec.~\ref{sec:spherical}, and its advantages relative to the usual Pl\"ucker coordinates are stated.
Finally, a model --that avoids the singularities of the spherical coordinates -- is proposed in Sec.~\ref{sec:newlinerep}, and its dynamics are computed.

\subsection{Pl\"ucker coordinates of a 3D Line}
\label{sec:pluckCoord}

3D straight lines are four degrees of freedom features, as explained in \cite{roberts1988}.
However, none of the current representations using four parameters can describe all 3D lines (including infinity), except the one presented in \cite{bartoli2005}. 
Nonetheless, this representation involves a QR factorization, and the computation of its dynamics is not straightforward.
In practice, many works use more degrees of freedom for that purpose. 
\emph{Pl\"ucker coordinates} is one of those non-minimal representations (see \cite{pottmann2009} for more detail). 
To cope with the up to scale definition of this parameterization, previous authors use a five-degree of freedom explicit representation using \emph{binormalized Pl\"ucker Coordinates} (see \cite{andreff2002}):
\begin{align}
    & \mathbfcal{L}  = \begin{bmatrix} \mathbf{d}\\ l\mathbf{m} \end{bmatrix},
    \label{eq:biline} \\
    & \text{with}\ \ \mathbf{m}^T\mathbf{d} = 0,
    \label{eq:ortho}
\end{align}
where $\mathbfcal{L}$ is the 3D line given by six coordinates of an element of the five-dimensional projective space $\mathcal{P}^5$, $\mathbf{d} \in \mathbb{R}^3$ is a unit vector representing the direction of the line, $\mathbf{m} \in \mathbb{R}^3$ denotes the line's unit moment vector, and $l \in \mathbb{R}$ is a depth parameter representing the geometric distance between the line and the correspondent reference frame (see Fig.~\ref{fig:line_proj}).

\begin{figure}
    \centering
    \def\svgwidth{.4\textwidth}
\begingroup%
  \makeatletter%
  \providecommand\color[2][]{%
    \errmessage{(Inkscape) Color is used for the text in Inkscape, but the package 'color.sty' is not loaded}%
    \renewcommand\color[2][]{}%
  }%
  \providecommand\transparent[1]{%
    \errmessage{(Inkscape) Transparency is used (non-zero) for the text in Inkscape, but the package 'transparent.sty' is not loaded}%
    \renewcommand\transparent[1]{}%
  }%
  \providecommand\rotatebox[2]{#2}%
  \newcommand*\fsize{\dimexpr\f@size pt\relax}%
  \newcommand*\lineheight[1]{\fontsize{\fsize}{#1\fsize}\selectfont}%
  \ifx\svgwidth\undefined%
    \setlength{\unitlength}{1559.05511811bp}%
    \ifx\svgscale\undefined%
      \relax%
    \else%
      \setlength{\unitlength}{\unitlength * \real{\svgscale}}%
    \fi%
  \else%
    \setlength{\unitlength}{\svgwidth}%
  \fi%
  \global\let\svgwidth\undefined%
  \global\let\svgscale\undefined%
  \makeatother%
  \begin{picture}(1,0.54545455)%
    \lineheight{1}%
    \setlength\tabcolsep{0pt}%
    \put(0,0){\includegraphics[width=\unitlength]{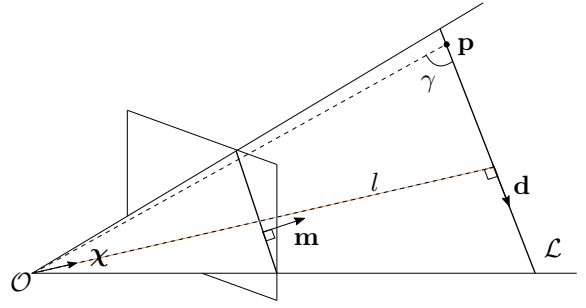}}%
    \put(0.88410574,0.19693214){\color[rgb]{0,0,0}\makebox(0,0)[lt]{\lineheight{1.25}\smash{\begin{tabular}[t]{l}$\mathbf{d}$\end{tabular}}}}%
    \put(0.48143206,0.1032579){\color[rgb]{0,0,0}\makebox(0,0)[lt]{\lineheight{1.25}\smash{\begin{tabular}[t]{l}$\mathbf{m}$\end{tabular}}}}%
    \put(0.62240907,0.20246037){\color[rgb]{0,0,0}\makebox(0,0)[lt]{\lineheight{1.25}\smash{\begin{tabular}[t]{l}$l$\end{tabular}}}}%
    \put(-0.03830684,0.01877018){\color[rgb]{0,0,0}\makebox(0,0)[lt]{\lineheight{1.25}\smash{\begin{tabular}[t]{l}$\mathcal{O}$\end{tabular}}}}%
    \put(0.93966288,0.08014577){\color[rgb]{0,0,0}\makebox(0,0)[lt]{\lineheight{1.25}\smash{\begin{tabular}[t]{l}$\mathcal{L}$\end{tabular}}}}%
    \put(0.10759063,0.07353824){\color[rgb]{0,0,0}\makebox(0,0)[lt]{\lineheight{1.25}\smash{\begin{tabular}[t]{l}$\bm{\chi}$\end{tabular}}}}%
    \put(0.71421355,0.39290774){\color[rgb]{0,0,0}\makebox(0,0)[lt]{\lineheight{1.25}\smash{\begin{tabular}[t]{l}$\gamma$\end{tabular}}}}%
    \put(0.78146825,0.45974107){\color[rgb]{0,0,0}\makebox(0,0)[lt]{\lineheight{1.25}\smash{\begin{tabular}[t]{l}$\mathbf{p}$\end{tabular}}}}%
  \end{picture}%
\endgroup%

    \caption{\it Projection of a 3D straight line ($\mathcal{L}$) onto a perspective image with optical center $\mathcal{O}$. The plane, which contains the line, and the optical center is the interpretation plane. The binormalized Pl\"ucker Coordinates are presented, with $\mathbf{d}$ being the direction vector of the line, $\mathbf{m}$ the moment vector, $l$ the line depth, and $\mathbf{p}$ is a point in the line. The vector $\bm{\chi}$ in \eqref{eq:chi} is the direction of the view ray associated with the closest point on the line to the optical center. Finally, $\gamma$ denotes the angle between the view ray and the line direction.}
    \label{fig:line_proj}
\end{figure}

Without loss of generality, setting the reference frame in the origin of the camera's optical center, we define the moment vector as parallel to the vector normal to the interpretation plane, which is defined by the 3D line and the origin (i.e., camera center). See \cite{hartley2003} for more detail about the underlying constraints. We then define
\begin{equation}
    \mathbf{m} = \frac{\mathbf{p} \times \mathbf{d}}{|| \mathbf{p} ||\sin(\gamma)},
    \label{eq:h_def}
\end{equation}
where $\mathbf{p} \in \mathbb{R}^3$ is any point in the line, and $\gamma$ denotes the angle between the view ray of the point and the direction vector.
The depth of the line is defined as:
\begin{equation}
    l = || \mathbf{p}|| \sin(\gamma),
\end{equation}
for any $\mathbf{p}$ on the line.
By definition, the line's image--in the normalized image plane--is given by the intersection of the interpretation plane with the image plane. 
The normalized moment vector $\mathbf{m}$ of the line can then be fully measured from the line image.

In this work, we aim to design observers to recover the full line state (\emph{Pl\"ucker coordinates} see \eqref{eq:biline}) from the measures $\mathbf{y} = \mathbf{m}$, and the camera velocities, which we assume to be known.
Now, we make the following assumptions:
\begin{assumption}
    The depth of the line must be positive at all times, i.e., $l > 0$.
    \label{assump:depth}
\end{assumption}
This practical assumption is based on the fact that, if $l = 0$, we reach a geometrical singularity. Specifically, the projection of a line is a single point in the image.

\begin{assumption}
    The intervals of time where the linear velocity is null or is contained in the line interpretation plane are bounded below and above by positive constants and do not occur infinitely often. 
    \label{assump:linearvel}
\end{assumption}
The non zero linear velocity is a known requirement for depth estimation, as shown in \cite{deluca2008}.
Besides, a visual dynamical system consists of the interaction matrix multiplied by the camera velocities, as shown in \cite{chaumette2006}.
Thus, those systems are affine with respect to the input.
If the inputs are identically zero, the state will be constant.
If the camera linear velocity belongs to the line interpretation plane, the system is not observable (details are shown in Sec.~\ref{sec:spicaMP}).
Assumption~\ref{assump:linearvel} is a very weak restriction and is meant to ensure that the periods of time --that the system is observable-- are dominant with respect to the periods that it is not. 
Technically, it will allow us to conclude that the error will converge to zero as $t\to\infty$. 
The behavior of the \emph{MHO} and \emph{MLO} -- when there are non observable time intervals -- is evaluated in Sec.~\ref{sec:sim}.

\subsection{3D Lines in Spherical Coordinates}
\label{sec:spherical}

A visual dynamical system is observable in the sense of typical observer design such as the one proposed in \cite{spica2013} if the matrix -- giving the influence of the unmeasurable state variables in the dynamics of the measurable ones -- is full column rank.
The previous condition is not verified for the dynamics of \emph{Pl\"ucker coordinates} presented in \cite{andreff2002}.
To address this issue, \cite{mateus2018} proposed a variable change to achieve an observable system if an additional algebraic constraint was enforced.
To prevent the need to solve a dynamic-algebraic system, we proposed a minimal line representation in \cite{mateus2019}.  
This representation is based on spherical coordinates and can be obtained from the \emph{Pl\"ucker coordinates} in two steps.
The first step consists of transforming $\mathbf{m}$ to spherical coordinates.
Since it is a unit vector --its length is one-- then we can represent it using only the spherical angles (see Fig.~\ref{fig:sphericalSteps}\subref{fig:step1}).
The second step consists of computing two vectors orthogonal to $\mathbf{m}$ and each other.
This allows us to construct an orthonormal matrix to which we project the direction vector $\mathbf{d}$ --previous scaled by the inverse line depth $l$ (see Fig.~\ref{fig:sphericalSteps}\subref{fig:step2}).

Using spherical coordinates, a line is defined using four parameters as
\begin{equation}
    \mathbfcal{L}_S = \begin{bmatrix} \theta \\ \phi \\ \eta_1 \\ \eta_2 \end{bmatrix},
    \label{eq:linesphere}
\end{equation}
where $-\pi \leq \theta \leq \pi$ is the azimuth, $-\frac{\pi}{2} \leq \phi \leq \frac{\pi}{2}$ is the zenith angle, $\eta_1$ and $\eta_2$ are the projection of the direction vector --scaled by the inverse depth-- onto the two remaining basis vectors.
Notice that spherical coordinates have two singularities associated with their discontinuity on the sphere's north and south poles.

\begin{figure}
    \centering
    \subfloat[Euclidean to spherical coordinates.]{
        \def\svgwidth{.19\textwidth}
\begingroup%
  \makeatletter%
  \providecommand\color[2][]{%
    \errmessage{(Inkscape) Color is used for the text in Inkscape, but the package 'color.sty' is not loaded}%
    \renewcommand\color[2][]{}%
  }%
  \providecommand\transparent[1]{%
    \errmessage{(Inkscape) Transparency is used (non-zero) for the text in Inkscape, but the package 'transparent.sty' is not loaded}%
    \renewcommand\transparent[1]{}%
  }%
  \providecommand\rotatebox[2]{#2}%
  \newcommand*\fsize{\dimexpr\f@size pt\relax}%
  \newcommand*\lineheight[1]{\fontsize{\fsize}{#1\fsize}\selectfont}%
  \ifx\svgwidth\undefined%
    \setlength{\unitlength}{300bp}%
    \ifx\svgscale\undefined%
      \relax%
    \else%
      \setlength{\unitlength}{\unitlength * \real{\svgscale}}%
    \fi%
  \else%
    \setlength{\unitlength}{\svgwidth}%
  \fi%
  \global\let\svgwidth\undefined%
  \global\let\svgscale\undefined%
  \makeatother%
  \begin{picture}(1,1.125)%
    \lineheight{1}%
    \setlength\tabcolsep{0pt}%
    \put(0,0){\includegraphics[width=\unitlength]{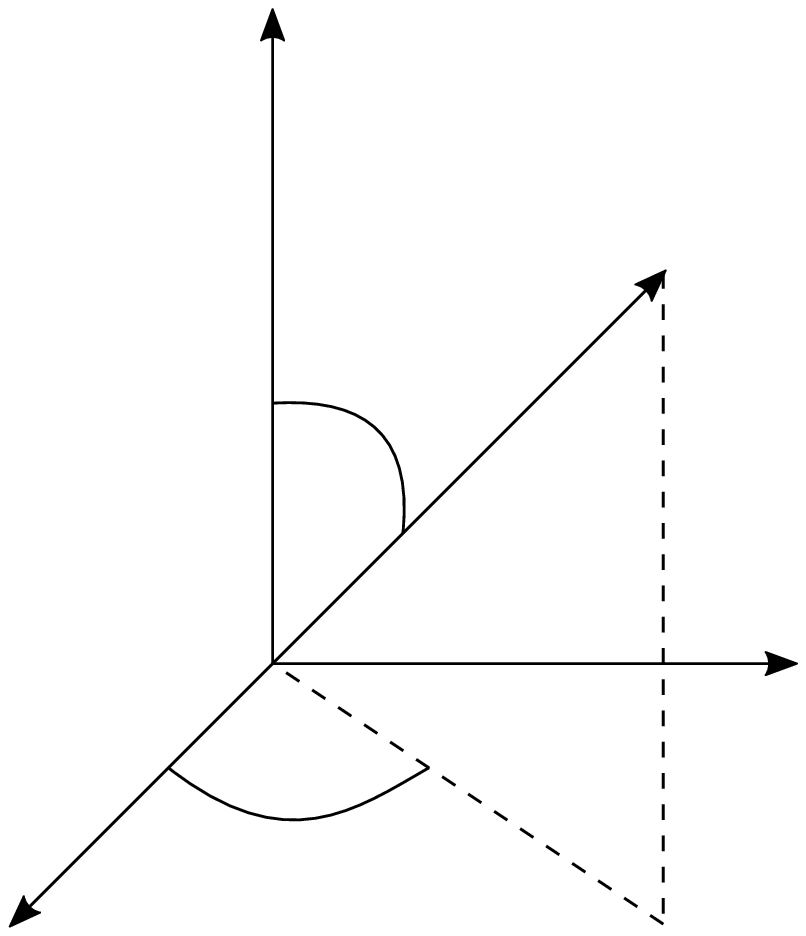}}%
    \put(0.33499996,1.17999787){\color[rgb]{0,0,0}\makebox(0,0)[lt]{\lineheight{1.25}\smash{\begin{tabular}[t]{l}$z$\end{tabular}}}}%
    \put(1.03499984,0.3549996){\color[rgb]{0,0,0}\makebox(0,0)[lt]{\lineheight{1.25}\smash{\begin{tabular}[t]{l}$y$\end{tabular}}}}%
    \put(-0.065,-0.02000184){\color[rgb]{0,0,0}\makebox(0,0)[lt]{\lineheight{1.25}\smash{\begin{tabular}[t]{l}$x$\end{tabular}}}}%
    \put(0.29999989,0.06499802){\color[rgb]{0,0,0}\makebox(0,0)[lt]{\lineheight{1.25}\smash{\begin{tabular}[t]{l}$\theta$\end{tabular}}}}%
    \put(0.45999997,0.65499989){\color[rgb]{0,0,0}\makebox(0,0)[lt]{\lineheight{1.25}\smash{\begin{tabular}[t]{l}$\phi$\end{tabular}}}}%
    \put(0.8500001,0.85999886){\color[rgb]{0,0,0}\makebox(0,0)[lt]{\lineheight{1.25}\smash{\begin{tabular}[t]{l}$\mathbf{m}$\end{tabular}}}}%
  \end{picture}%
\endgroup%

        \label{fig:step1}
    }
    \quad
    \subfloat[Projection to orthonormal basis.]{
        \def\svgwidth{.19\textwidth}
\begingroup%
  \makeatletter%
  \providecommand\color[2][]{%
    \errmessage{(Inkscape) Color is used for the text in Inkscape, but the package 'color.sty' is not loaded}%
    \renewcommand\color[2][]{}%
  }%
  \providecommand\transparent[1]{%
    \errmessage{(Inkscape) Transparency is used (non-zero) for the text in Inkscape, but the package 'transparent.sty' is not loaded}%
    \renewcommand\transparent[1]{}%
  }%
  \providecommand\rotatebox[2]{#2}%
  \newcommand*\fsize{\dimexpr\f@size pt\relax}%
  \newcommand*\lineheight[1]{\fontsize{\fsize}{#1\fsize}\selectfont}%
  \ifx\svgwidth\undefined%
    \setlength{\unitlength}{300bp}%
    \ifx\svgscale\undefined%
      \relax%
    \else%
      \setlength{\unitlength}{\unitlength * \real{\svgscale}}%
    \fi%
  \else%
    \setlength{\unitlength}{\svgwidth}%
  \fi%
  \global\let\svgwidth\undefined%
  \global\let\svgscale\undefined%
  \makeatother%
  \begin{picture}(1,1.125)%
    \lineheight{1}%
    \setlength\tabcolsep{0pt}%
    \put(0,0){\includegraphics[width=\unitlength]{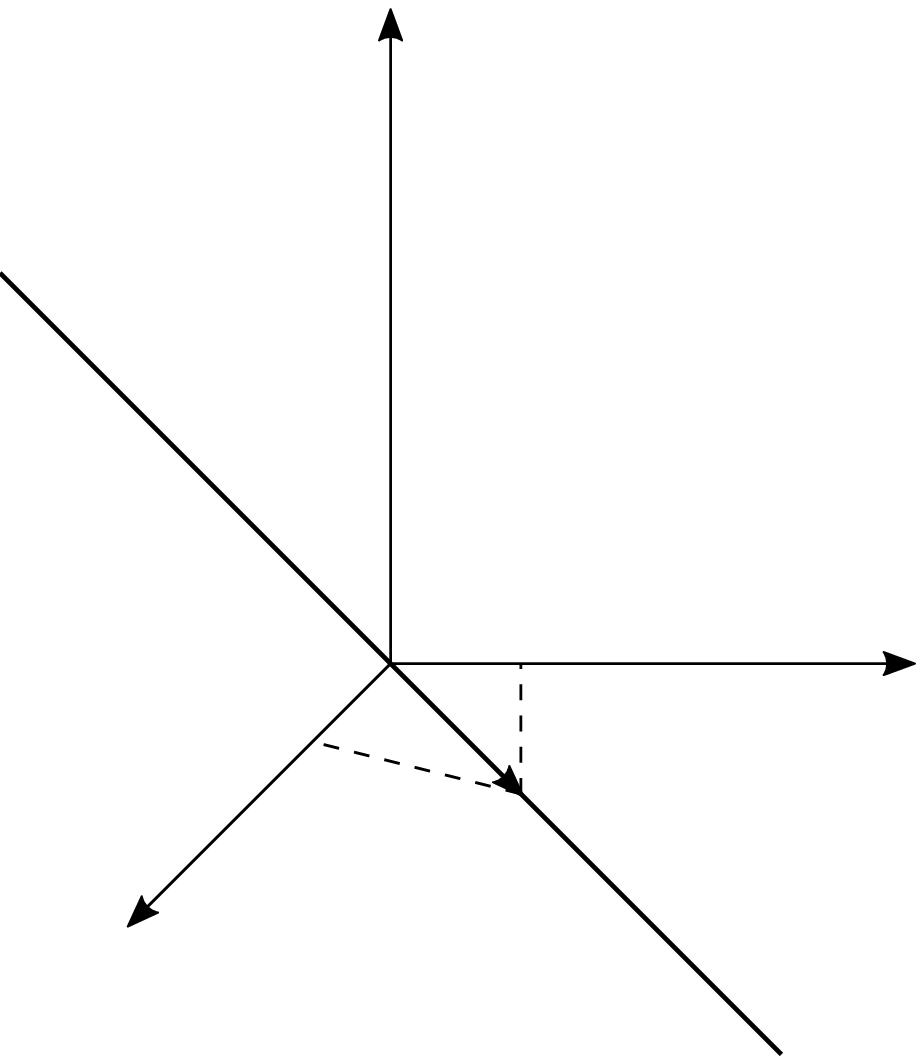}}%
    \put(0.33499996,1.17999787){\color[rgb]{0,0,0}\makebox(0,0)[lt]{\lineheight{1.25}\smash{\begin{tabular}[t]{l}$\mathbf{m}$\end{tabular}}}}%
    \put(0.78999975,0.43999932){\color[rgb]{0,0,0}\makebox(0,0)[lt]{\lineheight{1.25}\smash{\begin{tabular}[t]{l}$\mathbf{m}\times\mathbf{m}_p$\end{tabular}}}}%
    \put(-0,0.05499931){\color[rgb]{0,0,0}\makebox(0,0)[lt]{\lineheight{1.25}\smash{\begin{tabular}[t]{l}$\mathbf{m}_p$\end{tabular}}}}%
    \put(0.595,0.23500092){\color[rgb]{0,0,0}\makebox(0,0)[lt]{\lineheight{1.25}\smash{\begin{tabular}[t]{l}$\mathbf{d}$\end{tabular}}}}%
    \put(0.51499999,0.45999992){\color[rgb]{0,0,0}\makebox(0,0)[lt]{\lineheight{1.25}\smash{\begin{tabular}[t]{l}$\eta_2$\end{tabular}}}}%
    \put(0.18999997,0.34000052){\color[rgb]{0,0,0}\makebox(0,0)[lt]{\lineheight{1.25}\smash{\begin{tabular}[t]{l}$\eta_1$\end{tabular}}}}%
    \put(0.15,0.71000059){\color[rgb]{0,0,0}\makebox(0,0)[lt]{\lineheight{1.25}\smash{\begin{tabular}[t]{l}$\mathcal{L}$\end{tabular}}}}%
  \end{picture}%
\endgroup%

        \label{fig:step2}
    }
    \caption{\it Steps to obtain the minimal line modeling presented in \cite{mateus2019}. On the left, the transformation of the moment vector from Euclidean to spherical coordinates. Since the vector has a unit norm, we can define it using only the angles. On the right, the direction vector's projection to the orthonormal basis with the moment as one of the basis vectors. Given the orthogonality, in \eqref{eq:ortho} the projection yields two non-zero parameters.}
    \label{fig:sphericalSteps}
\end{figure}

The line dynamics are given as
\begin{align}
    \dot{\theta} & = \frac{-\bm{\omega}_c^T  \mathbf{m}_P  + \bm{\nu}_c^T\mathbf{m}_S\eta_1}{\cos(\phi)} \label{eq:dt}\\ 
    \dot{\phi} & = -\bm{\omega}_c^T (\mathbf{m}_S \times \mathbf{m}_P) + \bm{\nu}_c^T\mathbf{m}_S \eta_2 
    \label{eq:dp} \\ \nonumber
    \dot{\eta}_1 & = -\bm{\omega}_c^T \left( \mathbf{m}_P\tan(\phi) + \mathbf{m}_S \right)\eta_2 + \\ \nonumber
    & \ \ \ \ \ \ \ \ + \bm{\nu}_c^T \left(  \left(\mathbf{m}_S\tan(\phi) - \mathbf{m}_P \right)\eta_1\eta_2 + (\mathbf{m}_S\times\mathbf{m}_P)\eta_1^2\right) \\ 
    & = f_{\eta_1}(\theta,\phi,\eta_1,\eta_2, \bm{\nu}_c, \bm{\omega}_c)
   \label{eq:eta1_dyn} \\ \nonumber
    \dot{\eta}_2 & = \bm{\omega}_c^T\left( \mathbf{m}_P \tan(\phi) + \mathbf{m}_S \right)\eta_1 + \\ \nonumber
    & \ \ \ \ \ \ \ \ + \bm{\nu}_c^T \left( (\mathbf{m}_S\times\mathbf{m}_P)\eta_1\eta_2 - \mathbf{m}_S\tan(\phi)\eta_1^2 - \mathbf{m}_P\eta_2^2 \right) \\ 
    & = f_{\eta_2}(\theta,\phi,\eta_1,\eta_2, \bm{\nu}_c,\bm{\omega}_c)
    \label{eq:eta2_dyn},
\end{align}
with 
\begin{equation}
    \mathbf{m}_{S} = \begin{bmatrix} \cos(\theta) \cos(\phi)\\
                    \sin(\theta) \cos(\phi)\\
                    \sin(\phi)
                \end{bmatrix}
    \label{eq:hsph}
\end{equation}
and
\begin{equation}
    \mathbf{m}_P = \begin{bmatrix} \cos(\theta) \sin(\phi)\\
                    \sin(\theta) \sin(\phi)\\
                    - \cos(\phi)
                \end{bmatrix}
    \label{eq:hpsph}.
\end{equation}

Even though the previous dynamical system is observable, it has two singularities.
These can be easily found from inspection of \eqref{eq:dt}, and correspond to the poles of the sphere, more precisely when $\phi = \pm \frac{\pi}{2}$.
Similar singularities can be found in \cite{bartoli2005}, where Euler angles are used for representing rotation. 
The next section presents a novel representation that solves this problem. 

\subsection{Moment-Point Representation}
\label{sec:newlinerep}

This section presents a formulation that is both observable and does not suffer from the spherical coordinates singularities.
Let us define the vector $\bm{\chi}$ as
\begin{equation}
    \bm{\chi} = \frac{\mathbf{d} \times \mathbf{m}}{l}.
    \label{eq:chi}
\end{equation}
This vector represents the direction of the line point, which is closest to the camera's optical center, scaled by the inverse depth (in this case the depth of the point itself). 
A geometric representation of this vector is presented in Fig.~\ref{fig:line_proj}.
In this representation, a line is defined using six parameters as
\begin{equation}
    \mathbfcal{L}_{MP} = \begin{bmatrix} \mathbf{m} \\ \bm{\chi}  \end{bmatrix}.
    \label{eq:lineMP}
\end{equation}

Applying the transformation in \eqref{eq:chi} to the dynamics presented in \cite{andreff2002} the dynamics of a 3D line in the moment-point coordinates is given by
\begin{align}
        \dot{\mathbf{m}} & = \mathbf{g}_{\mathbf{m}}(\mathbf{m},\bm{\chi},\bm{\nu}_c,\bm{\omega}_c) =  [\bm{\omega}_c]_{\text{x}} \mathbf{m} - (\bm{\nu}_c^T\mathbf{m}) \bm{\chi} \label{eq:mdyn},\\ \nonumber
    \dot{\bm{\chi}} & =  \mathbf{g}_{\bm{\chi}}(\mathbf{m},\bm{\chi},\bm{\nu}_c,\bm{\omega}_c) \\ 
    & =  [\bm{\omega}_c]_{\text{x}} \bm{\chi} - (\bm{\nu}_c^T\mathbf{m}) \mathbf{m} \bm{\chi}^T \bm{\chi} + (\bm{\nu}_c^T\bm{\chi}) \bm{\chi} \label{eq:chidyn},
\end{align}
where $[\mathbf{m}]_{\text{x}}$ is the skew-symmetric that linearizes the cross product, such that $ \mathbf{m} \times \mathbf{d} = [\mathbf{m}]_{\text{x}}\mathbf{d}$.
Contrarily to \cite{andreff2002}, the influence of the unmeasurable variables is a full rank matrix in \eqref{eq:mdyn}, and the system is observable, for more details see Sec.~\ref{sec:spicaMP}.
\section{A Memory-free Observer for Dynamical Visual Systems}
\label{sec:obs_nonlinear}

In this work, the goal is to estimate the \emph{Pl\"ucker coordinates} of a line in 3D space. 
Since the dynamics of these coordinates, as presented in \cite{andreff2002}, result in an unobservable system, two alternative representations have been presented in Sec.~\ref{sec:line_dyn}, from which the \emph{Pl\"ucker coordinates} can be retrieved.
In this section, the design of observers as proposed in \cite{deluca2008}, and \cite{spica2013}, for both representations, are presented in Sec.~\ref{sec:spicaSphere}, and Sec.~\ref{sec:spicaMP} respectively.
Finally, the design of the gain matrices --of each observer-- are presented in Sec.~\ref{sec:designH}.

\subsection{Observer Design: Sphere}
\label{sec:spicaSphere}

This section presents an observer for the dynamical system proposed in \cite{mateus2019} and presented in Sec.~\ref{sec:spherical}. 
The observer defined as
\begin{align}
\begin{split}
    \begin{bmatrix} \dot{\hat{\theta}} \\ \dot{\hat{\phi}}\end{bmatrix} = & -\bm{\omega}_c^T \begin{bmatrix} \frac{\mathbf{m}_P}{\cos(\phi)} \\ (\mathbf{m}_S \times \mathbf{m}_P) \end{bmatrix}  + \bm{\Omega}_S^T  \begin{bmatrix} \hat{\eta_1} \\ \hat{\eta_2} \end{bmatrix} + \mathbf{H}_S \begin{bmatrix} \tilde{\theta} \\ \tilde{\phi} \end{bmatrix}\\
    \begin{bmatrix} \dot{\hat{\eta}}_1\\ \dot{\hat{\eta}}_2\end{bmatrix} = & \begin{bmatrix} f_{\eta_1}(\theta,\phi,\hat{\eta_1},\hat{\eta_2}, \bm{\nu}_c, \bm{\omega}_c) \\ f_{\eta_2}(\theta,\phi,\hat{\eta_1},\hat{\eta_2}, \bm{\nu}_c,\bm{\omega}_c) \end{bmatrix} + \alpha\bm{\Omega_S} \begin{bmatrix} \tilde{\theta} \\ \tilde{\phi} \end{bmatrix},
\end{split}
    \label{eq:spherical_observer}
\end{align}
with $\bm{\Omega}_S = \bm{\nu}_c^T\mathbf{m}_S \, \text{diag}(\cos(\phi)^{-1},1)$, $\mathbf{H}_S \succ 0$, $\tilde{\theta} = \theta - \hat{\theta}$, and $\tilde{\phi} = \phi - \hat{\phi}$, recovers the state in \eqref{eq:linesphere}, as long as the \emph{persistence of excitation condition} (see \cite{deluca2008,spica2013}) is verified. 
This condition holds if the matrix $\bm{\Omega}_S\bm{\Omega}_S^T$ is full rank.
Since it is a diagonal matrix, it will be full rank, as long as we are not in the poles of the sphere (see Sec.~\ref{sec:newlinerep}), and $\bm{\nu}_c^T\mathbf{m}_S \neq 0$.
For more details, see \cite{mateus2019}.

\subsection{Observer Design: M-P}
\label{sec:spicaMP}

This section presents an observer for the M-P line representation proposed in this work.
We followed the same rationale applied to the spherical coordinates.
From inspection of \eqref{eq:mdyn}, we have $\bm{\Omega}_{MP} = - \bm{\nu}_c^T\mathbf{m} \mathbf{I}_3$, which is also a diagonal matrix.
However, it does not present singularities --as in the previous case-- and it is identically zero, if and only if the linear velocity is null or is parallel to the line interpretation plane.
Thus, the system satisfies the observability criteria of \cite{deluca2008,spica2013}.
The observer for this system is given as
\begin{align}
    \begin{split}
        \dot{\hat{\mathbf{m}}} = & [\bm{\omega}_c]_{\text{x}}\mathbf{m} + \bm{\Omega}_{MP}^T\hat{\bm{\chi}} + \mathbf{H}_{MP} \tilde{\mathbf{m}}\\
        \dot{\hat{\bm{\chi}}} = & \mathbf{g}_{\bm{\chi}}(\mathbf{m},\hat{\bm{\chi}},\bm{\nu}_c,\bm{\omega}_c) + \alpha \bm{\Omega}_{MP}\tilde{\mathbf{m}},
    \end{split}
    \label{eq:newobserver}
\end{align}
where $\mathbf{H}_{MP} \succ 0$ and $\tilde{\mathbf{m}} = \mathbf{m} - \hat{\mathbf{m}}$.

\subsection{Gain matrices $\mathbf{H}_S$ and $\mathbf{H}_{MP}$}
\label{sec:designH}

The gain matrices $\mathbf{H}_S$ and $\mathbf{H}_{MP}$ are derived from their respective matrix $\bm{\Omega}$ as follows.
Let $\mathbf{U}\bm{\Sigma}\mathbf{V} = \bm{\Omega}$ be the singular value decomposition of the matrix $\bm{\Omega}$, where $\bm{\Sigma} = \text{diag}(\{\sigma_i\})$, $i = 1,..,m$, with $\sigma_i$ being the singular values from the lowest to the highest. 
Then, $\mathbf{H} \in \mathbb{R}^{m \times m}$ may be chosen as  
\begin{equation}
    \mathbf{H} = \mathbf{V} \mathbf{D} \mathbf{V}^T,
\end{equation}
where $\mathbf{D} \in \mathbb{R}^{m\times m}$ is a function of the singular values of $\bm{\Omega}$.
Following \cite{spica2013}, the former is defined as $\mathbf{D} = \text{diag}(\{c_i\})$, with $c_i > 0$, and $c_i = 2\sqrt{\alpha}\sigma_i$.
This choice prevents oscillatory modes, thus trying to  achieve a critically damped transient behavior.
Since both $\bm{\Omega}_S$ and $\bm{\Omega}_{MP}$ are diagonal, the design of the gain matrices is straightforward.
\section{Moving Horizon Observer for 3D Line Estimation}
\label{sec:mho_lines}

An MHO estimates the state by solving an optimization problem in a finite time window of size $N$.
The optimization window is moved each time a new measurement becomes available, and a new estimate is computed by solving the problem in the new window.
Since this approach is based on nonlinear optimization, constraints --on the system state and inputs-- can be added.
This section presents a Moving Horizon Observer for the visual dynamical system in Sec.~\ref{sec:newlinerep}. 
We start by presenting the discretization of that system and the design of the MHO optimization problem. 
Then, we present the \emph{MHO} assumption verification in Sec.~\ref{sec:assump_ver}. 
Finally, the stability analysis of the \emph{MHO} for 3D line estimation is presented in Sec.~\ref{sec:stab_analysis}.

\subsection{MHO Design}

We are particularly interested in the framework presented in \cite{alessandri2008}, where the observer consists of taking a set of $N$ measurements and inputs to minimize the error between the actual and expected measurements. 
The latter are obtained by the composition of the measurement model with the system dynamics $N-1$ times.
Besides the minimization of the previous cost, a prediction term is also introduced.
This prediction enforces the next estimate to be in close vicinity of the expected state, which results from applying the dynamics to the current state estimate.

The MHO presented in \cite{alessandri2008} applies to discrete nonlinear systems, but the system presented in Sec.~\ref{sec:newlinerep} is continuous.
Thus, we discretize the system in \eqref{eq:mdyn}, and \eqref{eq:chidyn} using the Euler method, yielding
\begin{multline}
    \mathbfcal{L}_{MP_{k+1}} = \begin{bmatrix} \mathbf{m}_{k+1} \\ \bm{\chi}_{k+1} \end{bmatrix} = \mathbf{f}(\mathbf{m}_k,\bm{\chi}_k,\bm{\nu}_{c_k},\bm{\omega}_{c_k}) = \\  \begin{bmatrix} \mathbf{m}_k \\ \bm{\chi}_k \end{bmatrix} + \begin{bmatrix} 
        \mathbf{g}_{\mathbf{m}}(\mathbf{m}_{k},\bm{\chi}_{k},\bm{\nu}_{c_{k}},\bm{\omega}_{c_{k}}) \\
        \mathbf{g}_{\bm{\chi}}(\mathbf{m}_{k},\bm{\chi}_{k},\bm{\nu}_{c_{k}},\bm{\omega}_{c_{k}})
    \end{bmatrix} \Delta t,
    \label{eq:sys_line_discrete}
\end{multline}
where $\Delta t$ is the time step, and $k \in \mathbb{N}$ denotes the time instant.
Since the measurements are provided by a camera --with a defined frame rate-- we take the time step to be the inverse of the camera frame rate.

The MHO framework accounts for both model and measurement noise, which are not present in our current model in \eqref{eq:sys_line_discrete}.
Let us consider the previous system with additive noise as
\begin{align}
    \mathbfcal{L}_{MP_{k+1}} = &~ \mathbf{f}(\mathbf{m}_k,\bm{\chi}_k,\bm{\nu}_{c_k},\bm{\omega}_{c_k}) + \bm{\xi}_{k}
    \label{eq:fmo_sys}
    \\
    \mathbf{y}_{k}  = &~ \mathbf{h}(\mathbf{\mathbfcal{L}_{MP_{k}}}) + \bm{\lambda}_{k} =  \mathbf{m} + \bm{\lambda}_{k},
    \label{eq:fmo_out}
\end{align}
where $\mathbf{y}_{k} \in \mathbb{R}^3$ is the output vector, $\bm{\xi}_{k}$ and $\bm{\lambda}_{k}$ are additive errors of the dynamics and measurement models, respectively.
As in \cite{alessandri2008}, the signals $\bm{\xi}_{k}$ and $\bm{\lambda}_{k}$ are assumed to be unknown, and deterministic variables with unknown statistics, which values belong to known compact sets.
Furthermore, the MHO requires a memory $\mathbfcal{I}_{k}$.
This memory contains past outputs and inputs, and is defined as
\begin{equation}
    \mathbfcal{I}_{k} = \{ \mathbf{y}_{k-N},...,\mathbf{y}_{k},\mathbf{u}_{k-N},...,\mathbf{u}_{k-1} \}. 
    \label{eq:fmo_mem}
\end{equation}
where $\mathbf{u}_{k} = [ \bm{\nu}_{c_k},\bm{\omega}_{c_k} ]^T$ is the input vector. 

Let $\hat{\mathbfcal{L}}_{MP_{k-N}},...,\hat{\mathbfcal{L}}_{MP_{k}}$ be the state estimates of $\mathbfcal{L}_{MP_{k-N}},...,\mathbfcal{L}_{MP_{k}}$, and $\overline{\mathbfcal{L}}_{MP_{k+1}}$ be a prediction of the next state, which is given by
\begin{equation}
    \overline{\mathbfcal{L}}_{MP_{k+1}} = \mathbf{f}(\hat{\mathbfcal{L}}_{MP_{k}},\mathbf{u}_{k}).
    \label{eq:x_pred}
\end{equation}
Given the memory in \eqref{eq:fmo_mem}, the output equation \eqref{eq:fmo_out}, and the prediction in \eqref{eq:x_pred}, a cost function can be defined as
\begin{multline}
    T(\hat{\mathbfcal{L}}_{MP_{k-N}},\mathbfcal{I}_{k}) = \mu \| \hat{\mathbfcal{L}}_{MP_{k-N}} - \overline{\mathbfcal{L}}_{MP_{k-N}}\|^2 + \\ \sum\limits_{i=k-N}^{k} \| \mathbf{y}_{i} - \mathbf{h}(\hat{\mathbfcal{L}}_{MP_{i}}) \|^2,
    \label{eq:mhoCost}
\end{multline}
where $\mu$ is a positive gain, and the state estimates $\hat{\mathbfcal{L}}_{MP_{k-N+1}},...,\hat{\mathbfcal{L}}_{MP_{k}}$ are computed using the system model in \eqref{eq:sys_line_discrete}.
The estimate at instant $k-N$ is given by
\begin{equation}
    \hat{\mathbfcal{L}}_{MP_{k-N}} = \argmin\limits_{\hat{\mathbfcal{L}}_{MP_{k-N}}} T(\hat{\mathbfcal{L}}_{MP_{k-N}},\mathbfcal{I}_{k}).
    \label{eq:mho_op}
\end{equation}
Notice that, the cost function in \eqref{eq:mhoCost} is not convex.
Thus, the optimization problem in \eqref{eq:mho_op} is not convex.
Nonetheless, it can be solved with state-of-the-art methods such as the method in \cite{lagarias1998}.

\subsection{Assumption Verification}
\label{sec:assump_ver}

The applicability of the MHO framework presented in \cite{alessandri2008} requires the satisfaction of four assumptions.
The first is concerned with the boundness of the noise signals and system inputs.
In practical applications, this assumption holds, i.e., the inputs of a robot and the measurement errors are bounded. 
The same applies to model errors.
Since the system's state will be bounded --for any bounded input sequence considered for the intended applications -- the second assumption also holds.

The third assumption is concerned with the continuity of the dynamics and measurement model. In particular, they should be twice differentiable continuous, i.e., $\mathcal{C}^2$.
The function $\mathbf{h}$ in \eqref{eq:fmo_out} is linear and thus is $\mathcal{C}^{\infty}$. 
Function $\mathbf{f}$ in \eqref{eq:fmo_sys} is smooth everywhere, except for $l = 0$. Given Assumption \ref{assump:depth}, $l>0$, and the continuity assumption holds.

The final assumption in \cite{alessandri2008} is concerned with the system observability.
The system must be observable in $N+1$ steps.
From \cite{alessandri2008} we know that the system is observable if the map 
\begin{equation}
    \mathbf{F}_N(\mathbfcal{L}_{MP_{k-N}},\overline{\mathbf{u}}) = \begin{bmatrix}
        \mathbf{h}(\mathbfcal{L}_{MP_{k-N}}) \\
        \mathbf{h} \circ \mathbf{f}(\mathbfcal{L}_{MP_{k-N}},\mathbf{u}_{k-N}) \\
        \vdots \\
        \mathbf{h} \circ \mathbf{f}^{\mathbf{u}_{k-1}} \circ ... \circ \mathbf{f}(\mathbfcal{L}_{MP_{k-N}},\mathbf{u}_{k-N})
    \end{bmatrix}.
    \label{eq:obsmap}
\end{equation}
where $\overline{\mathbf{u}} = \mathbf{u}_{k-N},...,\mathbf{u}_{k-1}$, is injective for all admissible inputs.
To assess injectivity of the map $\mathbf{F}$ let us introduce the following definitions:

\begin{definition}
    The system in \eqref{eq:fmo_sys} is said to be distinguishable for all possible states concerning all admissible inputs if $\forall \mathbf{x}_1, \mathbf{x}_2 \in X,\ \forall \overline{\mathbf{u}} \in U^N: \mathbf{x}_1 \neq \mathbf{x}_2 \implies \exists N_A: \mathbf{F}_{N_A}(\mathbf{x}_1,\overline{\mathbf{u}}) \neq \mathbf{F}_{N_A}(\mathbf{x}_2,\overline{\mathbf{u}})$.
    \label{def:disting}
\end{definition}
    
\begin{definition}
    The system in \eqref{eq:fmo_sys} is said to satisfy the observability rank condition for all possible states, concerninig all admissible inputs, if $\forall \mathbf{x} \in X, \forall \overline{\mathbf{u}} \in U^N, \exists N_O: \text{rank}(\frac{\partial \mathbf{F}_{N_O}}{\partial \mathbf{x}}) = n$.
    \label{def:rankobs}
\end{definition}

We can now introduce the following proposition:
\begin{proposition}
    If all the inputs in the memory $\mathcal{I}_k$ in \eqref{eq:fmo_mem} are such that the linear velocity components are non zero and are not contained in the line interpretation plane, the observation map $\mathbf{F}_N$ is injective. 
    \label{prop:injective}
\end{proposition}
\begin{proof}
    From \cite[Proposition 5]{hanba2009}, the observation map $\mathbf{F}_N$ is injective, if it is distinguishable and satisfies the observability rank condition.
    Let us start by verifying distinguishability.
    Recall that our state is composed of the moment vector $\mathbf{m}$, and the view ray of the line closest point to the optical center, scaled by the inverse depth $\bm{\chi}$.
    Let $\mathbf{x}_1 = [\mathbf{m}_1,\bm{\chi}_1]^T$, and $\mathbf{x}_2 = [\mathbf{m}_2,\bm{\chi}_2]^T$ be two different initial states.
    Furthermore, let us consider two different scenarios.
    The first is $\mathbf{m}_1 \neq \mathbf{m}_2$.
    For this case, it is easy to conclude from the dynamics in \eqref{eq:mdyn} that given the same input sequence to both systems, their outputs will always be different.
    The second scenario occurs when $\mathbf{m}_1 = \mathbf{m}_2$.
    In this case, the outputs will be distinguishable if and only if the linear velocity components are non zero and are not contained in the line interpretation plane.
    Let us consider that this is not the case.
    If $\bm{\nu} = 0$, the output depends only on the moment vector and angular velocity, which are equal, then the outputs cannot be distinguished.
    The same will occur if $\bm{\nu}_c^T\mathbf{m}_1 = \bm{\nu}_c^T\mathbf{m}_2 = 0$.
    
    Let us now evaluate the rank condition of Definition~\ref{def:rankobs}.
    The Jacobian of $\mathbf{h}$ is 
    \begin{equation}
        \frac{\partial \mathbf{h}}{\partial x} = \begin{bmatrix} \mathbf{I}_3  & \bm{0}_{3\times3} \end{bmatrix},
        \label{eq:dhdx}
    \end{equation}
    whose rank is $3$.
    The Jacobian of $\mathbf{h}(\mathbf{f(\mathbf{x}_{k},\mathbf{u}_{k})})$ is given by
    \begin{equation}
        \frac{\partial \mathbf{h}}{\partial \mathbf{x}} \frac{\partial \mathbf{f}}{\partial \mathbf{x}} (\mathbf{x}_{k},\mathbf{u}_{k}) = \begin{bmatrix} \mathbf{I}_3 + \bm{\rho}(\mathbf{x}_k,\mathbf{v}_k) & - \bm{\nu}^T\mathbf{m}\mathbf{I}_3 \end{bmatrix} \Delta t,
        \label{eq:dhfdx}
    \end{equation}
    where $\bm{\rho}(\mathbf{x}_k,\mathbf{v}_k)$ is a non zero function of the state and inputs.
    Stacking \eqref{eq:dhdx} and \eqref{eq:dhfdx} we get a $6 \times 6$ matrix, which is full rank, as long as the linear velocity is not zero and does not belong to the line interpretation plane. 
    Since it is a triangular inferior block matrix whose diagonal blocks consist of full rank diagonal matrices, the Jacobian is also full rank.
    Note that stacking additional $3\times6$ Jacobians will not reduce the overall rank of the full Jacobian matrix.
    Thus, the system verifies the observability rank condition for $N \geq 2$.

    Since the system is distinguishable and satisfies the observability rank condition, from \cite[Proposition 5]{hanba2009}, we conclude that the observation map $\mathbf{F}_N$ is injective. 
\end{proof}
Finally, since $\mathbf{F}_N$ is injective, the system is $S$-observable in $N+1$ steps, i.e., the assumption is verified.

\subsection{Stability Analysis}
\label{sec:stab_analysis}

In this section, the stability of the \emph{MHO} applied to the \emph{M-P} line representation is assessed.
In a first step, this analysis is based on the assumption that the system is observable for all instants of time $k$, i.e., the linear velocity is not null and is not contained in the line's interpretation plane.
If this is not the case, then from Assumption \ref{assump:linearvel}, the following analysis only holds for $k\ge k_0$, where $k_0$ is a finite time such that for all $k\ge k_0$, the system is observable. It then remains to show that for $k= 0, 1, \ldots k_0$ there is no finite escape, and therefore from \cite[Theorem 1]{alessandri2008} the error is bounded.

In \cite{hanba2010} it is shown the conditions of \cite[Theorem 1]{alessandri2008} are equivalent to $\mathbf{F}_N$ being injective and its Jacobian full column rank.
The function is injective as proven in Proposition~\ref{prop:injective}. Furthermore, its Jacobian is full rank if $N \geq 2$ --as long as the linear velocity throughout the observation windows is not zero, nor does it belong to the line interpretation plane.
Thus, from \cite[Theorem 1]{alessandri2008}, we conclude that the estimation error is bounded.

Our goal is to find an upper bound to the Lipschitz constant of $\mathbf{f}$, to find the range of $\mu$ that satisfy
\begin{equation}
    \frac{8 c_f^2 \mu}{\mu + \delta} < 1,
    \label{eq:u_cond}
\end{equation}
where $c_f$ is an upper bound of the Lipschitz constant of $\mathbf{f}$, and 
\begin{equation}
    \delta = \inf\limits_{\mathbfcal{L}_{MP_1},\mathbfcal{L}_{MP_2}, \mathbfcal{L}_{MP_1} \neq \mathbfcal{L}_{MP_2}} \frac{\varphi(\| \mathbfcal{L}_{MP_1} - \mathbfcal{L}_{MP_2} \|)}{\| \mathbfcal{L}_{MP_1} - \mathbfcal{L}_{MP_2} \|} > 0,
    \label{eq:t1}
\end{equation}
with $\mathcal{K}$-function $\varphi(.)$, such that
\begin{multline}
    \varphi(\| \mathbfcal{L}_{MP_1} - \mathbfcal{L}_{MP_2} \|) \leq \| \mathbf{F}_N(\mathbfcal{L}_{MP_1},\overline{\mathbf{u}}) - \mathbf{F}_N(\mathbfcal{L}_{MP_2},\overline{\mathbf{u}}) \|.
    \label{eq:k_obs}
\end{multline}
From the mean value theorem, we know that the Jacobian norm's upper bound is an upper bound to the Lipschitz constant.
The Jacobian of $\mathbf{g}(\mathbfcal{L}_{MP},\mathbf{u})$ is given by
\begin{equation}
    \mathbf{J}_{\mathbf{g}} = \begin{bmatrix} \mathbf{J}_1 & \mathbf{J}_2 \\ \mathbf{J}_3 & \mathbf{J}_4 \end{bmatrix},
\end{equation}
with
\begin{align}
    \begin{split}
        \mathbf{J}_1 = &~ [\bm{\omega}_c]_{\text{x}} - \bm{\chi}\bm{\nu}_c^T \\
        \mathbf{J}_2 = &~ -\bm{\nu}_c^T\mathbf{m}\mathbf{I}_3 \\
        \mathbf{J}_3 = &~ - \bm{\nu}_c^T\mathbf{m}\bm{\chi}^T\bm{\chi}\mathbf{I}_3 - \mathbf{m}\bm{\chi}^T\bm{\chi}\bm{\nu}_c^T \\
        \mathbf{J}_4 = &~ [\bm{\omega}_c]_{\text{x}} - 2\bm{\nu}_c^T\mathbf{m}\mathbf{m}\bm{\chi}^T + \bm{\chi}\bm{\nu}_c^T + \bm{\nu}_c^T\bm{\chi}\mathbf{I}_3.
    \end{split}
\end{align}
Recall that $\| \mathbf{m} \| = 1$, then we obtain
\begin{align}
    \begin{split}
        \|\mathbf{J}_1\| \leq &~ \|\bm{\omega}_c\| + \|\bm{\nu}_c\|\|\bm{\chi}\| \\
        \|\mathbf{J}_2\| \leq &~ \|\bm{\nu}_c\| \\
        \|\mathbf{J}_3\| \leq &~  2\|\bm{\nu}_c\|\|\bm{\chi}\|^2 \\
        \|\mathbf{J}_4\| \leq &~ \|\bm{\omega}_c\| + 4\|\bm{\nu}_c\|\|\bm{\chi}\|.
    \end{split}
\end{align}
If we consider $\mathbf{J}_{\mathbf{g}}$ to be the sum of each of its components with zero padding, and applying the triangular inequality, an upper bound to the norm is given as
\begin{equation}
    \|\mathbf{J}_{\mathbf{g}}\| \leq \sum_{i = 1}^4 \|\mathbf{J}_i\|.
\end{equation}
Thus an upper bound to the Lipschitz constant of $\mathbf{g(.)}$ is given by
\begin{equation}
    c_g = 2\|\bm{\omega}_c\| + \|\bm{\nu}_c\| + 5\|\bm{\nu}_c\|\|\bm{\chi}\| + 2 \|\bm{\nu}_c\| \|\bm{\chi}\|^2.
    \label{eq:kg}
\end{equation}

Nevertheless, we want to obtain an upper bound for the constant of $\mathbf{f}(.)$.
Keeping this goal in mind, we introduce the following lemma:
\begin{lemma}
The Lipschitz constant of a function $g(.)$ discretized using the Euler method becomes $c = 1 + c_g \Delta t $, with $c_g$ being the Lipschitz constant of $g(.)$ and $\Delta t$ the sampling time.
\label{lemma:euler_lip}
\end{lemma}
\begin{proof}
    Let $f(x)$ be the discretization of $g(x)$ by the Euler method, then
    \begin{equation}
        f(x) = x + g(x)\Delta t.
    \end{equation}
    Let $y$ be another point in the functions domain, then
    \begin{equation}
        \| f(x) - f(y) \| = \| x + g(x) \Delta t - y - g(y) \Delta t \|.
        \label{eq:lip1}
    \end{equation}
    Applying the triangular inequality to the right handside of \eqref{eq:lip1} yields
    \begin{equation}
        \| f(x) - f(y) \| \leq \| x - y\| + \|g(x)-g(y)\|\Delta t.
        \label{eq:lip2}
    \end{equation}
    Now, applying the Lipschitz condition of $g(.)$, we obtain
    \begin{equation}
        \| f(x) - f(y) \| \leq \| x - y\| + c_g\| x - y\|\Delta t = (1+c_g\Delta t)\| x - y\|.
    \end{equation}
\end{proof}

Thus, an upper bound of the Lipschitz constant of $\mathbf{f}(\mathbf{x}_{k},\mathbf{u}_{k})$ with respect to $\mathbf{x}$ can be written as
\begin{equation}
    c_f = 1 + c_g\Delta t.
    \label{eq:cf}
\end{equation}
From Lemma~\ref{lemma:euler_lip} it is trivial to conclude that $8c_f^2 > 1$, and thus \eqref{eq:u_cond} cannot be verified for all $\mu >0$.
Thus, $\mu < \delta/(8c_f^2-1)$, and the parameter $\delta$ must be computed as well.

Let us introduce the following definition
\begin{definition}
    If a function $\mathbf{f}(\mathbf{x},\mathbf{u}) : X \subset \mathbb{R}^n \times V \subset \mathbb{R}^p  \rightarrow \mathbb{R}^n$ verifies
    \begin{multline}
        \frac{1}{L} \| \mathbf{x}_1 - \mathbf{x}_2\| \leq
        \| \mathbf{f}(\mathbf{x}_1,\mathbf{u}) - \mathbf{f}(\mathbf{x}_2,\mathbf{u}) \| \leq L \| \mathbf{x}_1 - \mathbf{x}_2\|, \\ \forall \mathbf{x}_1,\mathbf{x}_2 \in D,
    \end{multline}
    for $L > 1$, it is called bi-Lipschitz.
    \label{def:biLip}
\end{definition}
Furthermore, for a bi-Lipschitz function the following Lemma holds:
\begin{lemma}[from \cite{weaver1999}]
    A mapping $\mathbf{F}: X \rightarrow Y$ is said to be bi-Lipschitz if it is Lipschitz, and its inverse is also Lipschitz continuous.
    \label{lemma:biLip}
\end{lemma}

We can now state the following proposition:
\begin{proposition}
    If all the inputs in the memory $\mathcal{I}_k$ in \eqref{eq:fmo_mem} are such that the linear velocity components are non zero and are not contained in the line interpretation plane, the map $\mathbf{F}_N$ in \eqref{eq:obsmap} is bi-Lipschitz for the system in \eqref{eq:sys_line_discrete}. Furthermore, $\delta = \frac{1}{c_F}$, where $c_F$ is an upper bound to the map's Lipschitz constant.
    \label{prop:bilipDelta}
\end{proposition}
\begin{proof}
    From Proposition~\ref{prop:injective}, $\mathbf{F}_N$ is injective, i.e., it admits a left inverse. 
    However, to be invertible, we need to verify that it is also surjective.
    The function $\mathbf{h}$ is onto since it is a linear map.
    The function $\mathbf{f}$ is surjective, as applying the dynamics to a state generates a new state in the domain, i.e., the codomain is equal to the domain.
    Since the composition of surjective functions is surjective, and $\mathbf{F}_N$ consists of the stacking of surjective functions it is onto.
    Thus, $\mathbf{F}$ is bijective (injective and surjective) and is invertible.
    Given that the map is applied to a compact set of possible states, from \cite[Corolary 1.5.4]{weaver1999}, the inverse of $\mathbf{F}$ is Lipschitz.
    Thus, from Lemma~\ref{lemma:biLip}, $\mathbf{F}$ is bi-Lipschitz.
    
    From Definition~\ref{def:biLip} and \eqref{eq:k_obs} we define
    \begin{equation}
        \varphi(\| \mathbfcal{L}_{MP_1} - \mathbfcal{L}_{MP_2} \|) = \frac{1}{c_F} \| \mathbfcal{L}_{MP_1} - \mathbfcal{L}_{MP_2} \|,
        \label{eq:varphi}
    \end{equation}
    which is a $\mathcal{K}$-function.
    Then, $\delta$ can be easily computed from \eqref{eq:varphi} and \eqref{eq:t1}, yielding $\delta = \frac{1}{c_F}$. 
\end{proof}

To compute $\delta$, we need to compute the Lipschitz constant of $\mathbf{F}_N$.
Let us first get the Lipschitz constant of each element of $\mathbf{F}_N$.
Since $\mathbf{h}$ is linear, its Lipschitz constant is one.
The second element is the composition of $\mathbf{f}$ with $\mathbf{h}$, since the constant of the latter is one, the Lipschitz constant of $\mathbf{h(\mathbf{f(.)})}$ is $c_f$.
The remaining elements consist of further compositions of the system dynamics with the measurement function.
Thus, for $N-1$ compositions, the constant is $c_f^{N-1}$.
From Lemma~\ref{lemma:euler_lip}, we have that $c_f > 1$. Thus the Lipschitz constant increases with the size of the estimation window $N$.
Then, making use of the triangular inequality, we obtain
\begin{equation}
    c_F = \sum_{k=1}^{N} c_f^{k-1}.
    \label{eq:kF}
\end{equation}
Given the Lipschitz constant of $\mathbf{F}$ we can compute $\delta$, and consequently compute the range of values for the gain $\mu$, which satisfy \eqref{eq:u_cond}.

Finally, we can state the main result of this section:
\begin{theorem}
    If Propositions~\ref{prop:injective} and \ref{prop:bilipDelta} are verified, then the squared norm of the estimation error of the \emph{MHO} is bounded by a sequence $\zeta_k$.
    Furthermore, if $\mu < 1/(c_F(8c_f^2-1))$, the sequence $\zeta_k$ converges exponentially to an asymptotic value $e_{\infty}(\mu)$, and it is decreasing.
    \label{theo:bounded}
\end{theorem}
\begin{proof}
    From \cite[Theorem 1]{alessandri2008} if the system is observable in $N+1$ steps, and \eqref{eq:t1} holds, then the the estimation error is bounded by a sequence {$\zeta_k$}.
    Since Propositions~\ref{prop:injective} is verified, the system is observable.
    Besides, from \cite{hanba2010} it can be concluded that, if $\mathbf{F}_N$ is injective and its Jacobian full column rank, then \eqref{eq:t1} is verified.
    Thus, the estimation error is bounded.
    Furthermore, from Proposition~\ref{prop:bilipDelta}, $\delta = 1/c_F$, then from \eqref{eq:u_cond}, $\mu < 1/(c_F(8c_f^2-1))$, and by \cite[Theorem 1]{alessandri2008} the sequence $\zeta_k$ converges exponentially to an asymptotic value $e_{\infty}(\mu)$, and it is decreasing.
\end{proof}

The previous theorem shows that the system estimation error is bounded if it is observable for all $k$.
To show that this still holds in cases that the system is not observable for some periods of time, we present the following result:
\begin{theorem}
    Under Assumptions \ref{assump:depth}, \ref{assump:linearvel}, and if the noise and inputs of the system are bounded, there is a finite positive time $k_0$ such that for $k = 0,1, \ldots k_0$, the estimation error of the \emph{MHO} is bounded. Moreover, for $k> k_0$, the system is observable and Theorem~\ref{theo:bounded} applies.
\end{theorem}
\begin{proof}
    Two conditions lead the \emph{M-P} system not to be observable.
    These correspond to the linear velocity being either null or parallel to the line's interpretation plane.
    Let us consider, the case $\bm{\nu} = 0$.
    The system dynamics become
    \begin{align}
        \begin{split}
            \mathbf{m}_{k+1} & = \mathbf{m}_{k} + (\bm{\omega}_{c_k} \times \mathbf{m}_k) \Delta t\\
            \bm{\chi}_{k+1} & = \bm{\chi}_{k} + (\bm{\omega}_{c_k} \times \bm{\chi}_k) \Delta t.
        \end{split}
        \label{eq:dynv0}
    \end{align}
    Taking the difference between consecutive states, and applying the Euclidean norm, we obtain
    \begin{align}
        \left\| \begin{bmatrix} \mathbf{m}_{k+1} - \mathbf{m}_{k} \\ \bm{\chi}_{k+1} - \bm{\chi}_{k} \end{bmatrix} \right\| \le &~ \|\bm{\omega}_{c_k}\|(1 + \|\bm{\chi}_k\|)\Delta t,
        \label{eq:stateupdatev0}
    \end{align}
    recall that $\| \mathbf{m}_k \| = 1, \forall k$.
    The right-hand side of \eqref{eq:stateupdatev0} is finite, and taking the summation over the time period $k = 0,1, \ldots k_0$, will yield a finite value.
    Let us now consider the case $\bm{\nu}_c^T\mathbf{m} = 0$.
    The system dynamics become
    \begin{align}
        \begin{split}
            \mathbf{m}_{k+1} & = \mathbf{m}_{k} + \bm{\omega}_{c_k} \times \mathbf{m}_k \Delta t\\
            \bm{\chi}_{k+1} & = \bm{\chi}_{k} + \left(\bm{\omega}_{c_k} \times \bm{\chi}_k + \bm{\nu}_{c_k}^T\bm{\chi}_k \bm{\chi}_k \right) \Delta t.
        \end{split}
        \label{eq:dynvmnull}
    \end{align}
    The state update is bounded by
    \begin{align}
        \left\| \begin{bmatrix} \mathbf{m}_{k+1} - \mathbf{m}_{k} \\ \bm{\chi}_{k+1} - \bm{\chi}_{k} \end{bmatrix} \right\| \le &~ \large(\|\bm{\omega}_{c_k}\|(1 + \|\bm{\chi}_k\|) + \|\bm{\nu}_{c_k}\| \|\bm{\chi}_k\|^2 \large)\Delta t.
        \label{eq:stateupdatevmnull}
    \end{align}
    Once again the summation over the time interval will also be finite.
    
    From \eqref{eq:fmo_out}, the measurements correspond to the moment vector. 
    Besides, from \eqref{eq:dynv0}, and \eqref{eq:dynvmnull}, the dynamics of that vector depend only on the angular velocity and itself.
    Thus, the cost function in \eqref{eq:mhoCost} is upper bounded by a function of the norms of the angular velocity, the measurement error $\bm{\lambda}_k$, and model error $\bm{\xi}_k$, which are all bounded.
    Given that the cost function is bounded, for all initial states, and from Assumption~\ref{assump:depth}, $\|\bm{\chi}\|$ cannot be infinite, and then the state estimate will be bounded.
    Since both the state and the estimate are bounded, so is the state estimation error.
    
    Finally, for $k > k_0$ the system is observable, i.e., the linear velocity is neither null nor orthogonal to $\mathbf{m}$, then Proposition~\ref{prop:injective}, and \ref{prop:bilipDelta} are verified, and Theorem~\ref{theo:bounded} can be applied.
\end{proof}
\begin{table}[t]
    \vspace{.1cm} 
    \centering
    \resizebox{.48\textwidth}{!}{
    \begin{tabular}{|c|c|c|c|c|c|c|}
        \hline
        $N$ & 2 & 3 & 4 & 5 & 6 & 7 \\ \hline 
        $\delta$ & 0.484 & 0.312 & 0.226 & 0.175 & 0.141 & 0.116 \\ \hline
        $\mu$ & 0.060 & 0.038 & 0.028 & 0.022 & 0.017 & 0.014 \\
        \hline
        \end{tabular}
    }
    \caption{\it Values of $\delta$ and $\mu$ for different estimation window sizes, and considering that $\max \|\bm{\nu}_c\| = \max \| \bm{\omega}_c\| = 0.5$, and $\max \|\bm{\chi}\| = 0.2$.}
    \label{tab:delta_mu}
\end{table}

\section{Experimental Results}
\label{sec:expResults}

This section presents the evaluation of the \emph{MHO} in Sec.~\ref{sec:mho_lines}, and \emph{MLO} in Sec.~\ref{sec:obs_nonlinear}, for state estimation of the line representations \emph{Sphere} (see Sec.~\ref{sec:spherical}), and \emph{M-P} (see Sec.~\ref{sec:newlinerep}). 
The evaluation consists of simulated (see Sec.~\ref{sec:sim}) and real data (see Sec.~\ref{sec:realResults}).
In simulation four representation/observer combinations are considered: 1) \emph{MLO+Sphere}; 2) \emph{MLO+M-P}, 3) \emph{MHO+M-P}, and 4) \emph{MHO+Sphere}. 
It is important to stress that there are no guarantees of convergence when the state approaches the singularities of the \emph{Sphere} representation. Nevertheless, the method can converge if the camera motion keeps the state far enough from the singularities.
The focus of the evaluation is on the converge time and robustness to measurement errors.
The real data was acquired by a mobile robot MBOT (see Fig.~\ref{fig:mbotSetup}\subref{fig:mbot}) and a manipulator of the Baxter robot (see Fig.~\ref{fig:baxterSetup}).
Since the Sphere representation has singularities, and thus no convergence guarantees can be provided, only the methods \emph{MLO+M-P}, and \emph{MHO+M-P} were evaluated with real data.  

\subsection{Simulation Results}
\label{sec:sim}

We consider four different tests.
In the first, we evaluate the runtime of each iteration of the \emph{MHO+M-P}, to assess if the observer can be applied to real-time applications.
The second test aims at evaluating the \emph{MHO} against the \emph{MLO} in noise-free scenarios.
The goal is to assess the convergence time of both observers with different parameters.
For the \emph{MLO}, the gain $\alpha$ (in \eqref{eq:newobserver}) was varied.
For the \emph{MHO}, the size $N$ (observation window), and the weight $\mu$ (in \eqref{eq:mhoCost}) were varied.
One parameter was fixed, and the other is varied and vice-versa.
The third test evaluates the robustness of the methods to measurement noise. 
The observers' window size and gains were chosen such that the convergence time is as similar as possible.
The final test evaluates the behavior of both the \emph{MLO} and \emph{MHO} when there are times intervals where the system is not observable.
The goal is to assess the performance of the observers and if they are still able to converge.

The simulations were conducted in MATLAB to generate the lines and run the observers.
The \emph{MHO} was implemented using the \emph{Optimization Toolbox}, in particular the \emph{fminsearch} function, which uses the simplex method of \cite{lagarias1998}. 
This method is not guaranteed to converge.
Thus, to assess if the method converged given our cost function, we run 200 simulations of 3D line estimation (10 seconds each), amounting to approximately 58000 calls to fminsearch, and evaluated the exit flag, which returned a true statement for all of the function calls.  
The imaging device consists of a perspective camera, with intrinsic parameters (for more details see \cite{hartley2003}) defined as $\mathbf{I}_3 \in \mathbb{R}^{3\times3}$. 
The camera is assumed to be free-flying; i.e., it can move in any of its six degrees-of-freedom.
To generate a 3D line, we start by sampling a point in a $5$ meter side cube in front of the camera.
Then, a unit direction $\mathbf{d}$ is randomly selected, and the moment vector $\mathbf{m}$ is computed using \eqref{eq:h_def}.
The initial state is initialized with the real values for $\mathbf{m}$, $\theta$, and $\phi$, and the unknown quantities $\bm{\chi}$, $\eta_1$, and $\eta_2$ are set randomly.
The camera linear velocities are computed using the Active Vision control law proposed in \cite{spica2013}, which has been tuned for lines in \cite{mateus2018,mateus2019}.

To assess the range of values the gain $\mu$ in \eqref{eq:mhoCost} can take for different observation window sizes, we start by selecting the camera's operation range velocities and the minimum distance a line can be to the camera.
These ranges are $\max \|\bm{\nu}_c\| = \max \| \bm{\omega}_c\| = 0.5$, and $\max \|\bm{\chi}\| = 0.2$.
To compute an upper bound to $\mu$, we start by computing $c_F$ in \eqref{eq:kF}, then the parameter $\delta$ defined in Proposition~\ref{prop:bilipDelta}, and finally the upper bound using \eqref{eq:u_cond}.
The values of $\delta$ and $\mu$ for different estimation window sizes are presented in Tab.~\ref{tab:delta_mu}.

\vspace{.15cm}
\noindent
{\bf Iteration time:} 
This test presents the evaluation of the computation time of the iterations of the four methods.
To achieve this goal, the runtime of each iteration of $100$ runs were stored for different window sizes of the \emph{MHO} methods.
For the methods using the \emph{MLO}, the iteration time of $100$ runs was also considered.
Notice that the gain $\alpha$ does not influence the runtime.
Tab.~\ref{tab:iterRunTIme} presents the mean and median iteration runtime.
The most computational effort required in the \emph{MLO} is an SVD decomposition, and thus its iteration time is considerably shorter than the \emph{MHO}.
From inspection, we see that even for greater window sizes, the \emph{MHO} can be applied in real-time applications; since it can still run at $27$ frames-per-second (fps).
Notice that typical cameras operate at $24$ or $30$ fps.
With respect to the \emph{Sphere} and \emph{M-P} representations, one iteration of the \emph{MHO} is faster for the former, which is due to the smaller state space.
The higher iteration time for the \emph{MLO+Sphere} with respect to the \emph{MLO+M-P} can be explained by the multiple calls to the \emph{sin} and \emph{cos} functions.

Iteration time was also evaluated for a different number of lines for the \emph{MLO+M-P} and \emph{MHO+M-P} methods.
For this test, the state space was increased by stacking the \emph{M-P} coordinates of each line.
Thus, the state's size is $6M$, with $M$ being the number of lines.
The results are presented in Fig.~\ref{fig:iterTimeMulti}\subref{fig:mloMulti} for the \emph{MLO+M-P} method, and in Fig.~\ref{fig:iterTimeMulti}\subref{fig:mhoMulti} for \emph{MHO+M-P}.
From one to two lines, the iteration times raises drastically, about one order of magnitude.
When considering three lines, the iteration time tends to double, but the growth rate slows down when adding more lines.
The ability to estimate multiple lines at a typical camera framerate is an issue, which must be addressed.
Solutions to the increase of the iteration time are: 
1) optimizing the solver by either using implementations in C++ and/or designing a specific solver to the optimization problem;
2) to have several estimators in parallel, but the available CPU threads would limit the number of lines; and
3) to perform the estimation in multi-robot systems, thus distributing the computational burden.

\begin{table}[t]
    \vspace{.1cm} 
    \centering
    \subfloat[Iteration time for both observers using the Sphere representation]{
        \resizebox{.5\textwidth}{!}{
        \begin{tabular}{|c|c|c|c|c|c|c|}
            \cline{2-7}
            \multicolumn{1}{c|}{} & \multicolumn{5}{c|}{MHO, N} & {MLO} \\
            \hline
            {\bf Time/Iter} (sec) & 3 & 4 & 5 & 6 & 7 &  \\
            \hline
            Mean & 0.020 & 0.024 & 0.027 & 0.032 & 0.036 & 1.761$\times10^{-4}$ \\
            \hline
            Median & 0.019 & 0.023 & 0.026 & 0.029 & 0.033 & 1.800$\times10^{-4}$ \\
            \hline
        \end{tabular}
        }
    }\\
    \subfloat[Iteration time for the observers using the M-P representation]{
        \resizebox{.5\textwidth}{!}{
        \begin{tabular}{|c|c|c|c|c|c|c|}
            \cline{2-7}
            \multicolumn{1}{c|}{} & \multicolumn{5}{c|}{MHO, N} & {MLO} \\
            \hline
            {\bf Time/Iter} (sec) & 3 & 4 & 5 & 6 & 7 &  \\
            \hline
            Mean & 0.023 & 0.027 & 0.031 & 0.034 & 0.037 & 7.375$\times10^{-5}$  \\
            \hline
            Median & 0.022 & 0.026 & 0.029 & 0.033 & 0.036 & 6.800$\times10^{-5}$ \\
            \hline
        \end{tabular}
        }
    }
    
    \caption{\it Mean and median runtimes of MHO and MLO iterations for both line representations. An iteration is defined as solving one instance of the optimization problem in \eqref{eq:mho_op} for the MHO. For the MLO it is defined as time to compute the gain matrix $H$ in Sec.~\ref{sec:designH} and the integration of the MLO dynamics (see Sec.~\ref{sec:spicaMP} and Sec.~\ref{sec:spicaSphere}).
    }
    \label{tab:iterRunTIme}
\end{table}

\begin{figure}
    \centering
    \subfloat[Iteration time for the \emph{MLO+M-P} method with multiple lines.]{
        \includegraphics[width=0.23\textwidth]{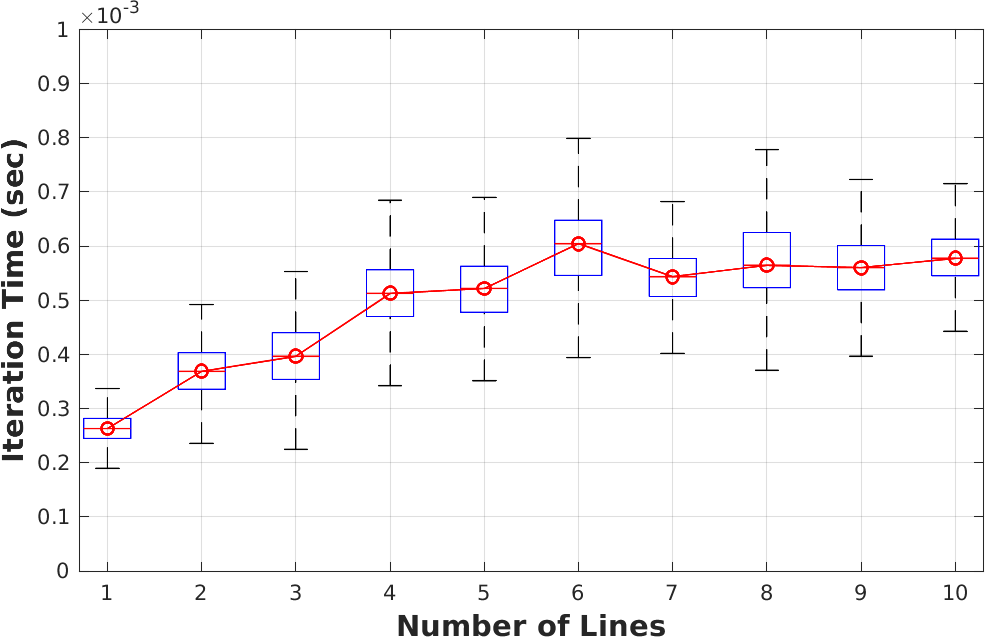}
        \label{fig:mloMulti}
    }\hfill
    \subfloat[Iteration time for the \emph{MHO+M-P} method with multiple lines.]{
        \includegraphics[width=0.23\textwidth]{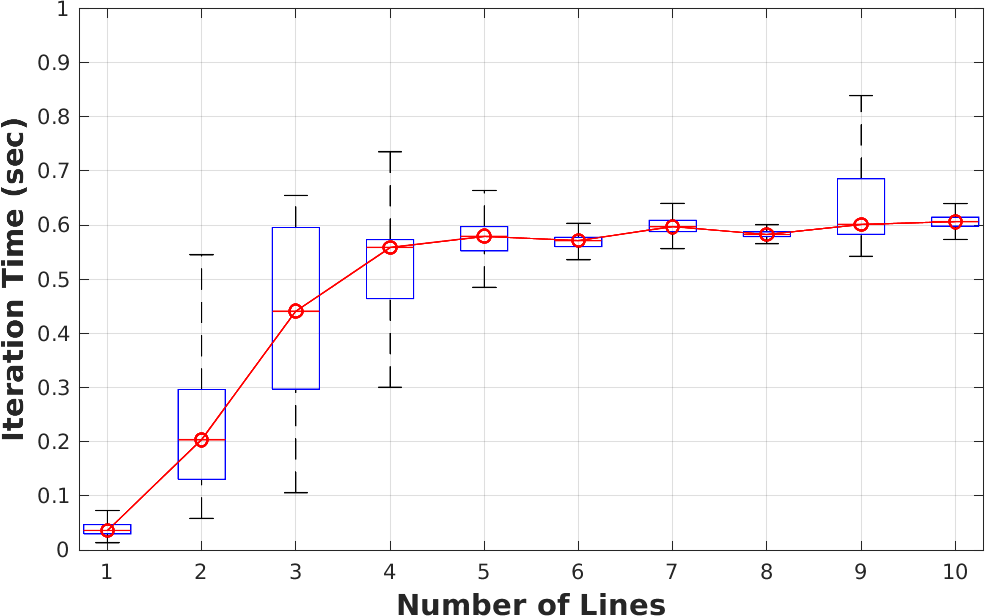}
        \label{fig:mhoMulti}
    }
    \caption{\it Iteration runtimes of 100 runs of the \emph{MLO+M-P} (left) and \emph{MHO+M-P} (right) with an increasing number of multiple lines. Notice that the vertical scales of the plots is different.}
    \label{fig:iterTimeMulti}
\end{figure}

\vspace{.15cm}
\noindent
{\bf Noise free data:}
This simulation test consists of running the 1) \emph{MLO+Sphere}, 2) \emph{MLO+M-P}, 3) \emph{MLO+Sphere}, and 4) \emph{MHO+M-P} for $100$ runs in a noiseless scenario. 
The goal is to evaluate the mean and median convergence time\footnote{We consider that the method converged when the norm of the estimation error is less than $0.01$.} of both observers for different values of their parameters. 
Tab.~\ref{tab:noiseFree} present the results of this test. In the table, we can see the convergence times for different parameter combinations. 
The fastest convergence is achieved for $\alpha = 1000$ for the \emph{MLO+Sphere}, closely followed by the \emph{MLO+M-P} and \emph{MHO+Sphere} for $\alpha = 1000$, and $N = 7$, $\mu = 0.014$ respectively. 
Since the convergence is assessed with respect to the norm of the estimation error, it is expected that the \emph{Sphere} representation presents a faster convergence.
Recall that the state space is smaller than the one of the \emph{M-P} representation.
Nevertheless, the price to pay for the faster convergence of the \emph{Sphere} is the observers' stability due to this representation's singularities.
An illustration of each parameter's effect in the convergence time is presented in Fig.~\ref{fig:simNoiselessResults}. 
We can see that the higher the gain of the \emph{MLO}, the faster it converges.
As for the \emph{MHO}, the wider the observation window, the faster it converges, and the higher the gain, the slower the observer converges.
The effect of the gain in the \emph{MHO} convergence is expected from inspection of the cost function in \eqref{eq:mhoCost}.
The higher the gain, the more weight the prediction term has, and thus the next estimate will be closer to the prediction.
Meaning that if the estimate is still far from the real value, it will tend to stay closer to the prediction than to move towards the state of minimum measurement error.

\begin{figure*}
    \centering
    \subfloat[Convergence of the \emph{MLO+M-P} method for increasing gains.]{
        \includegraphics[width=0.3\textwidth]{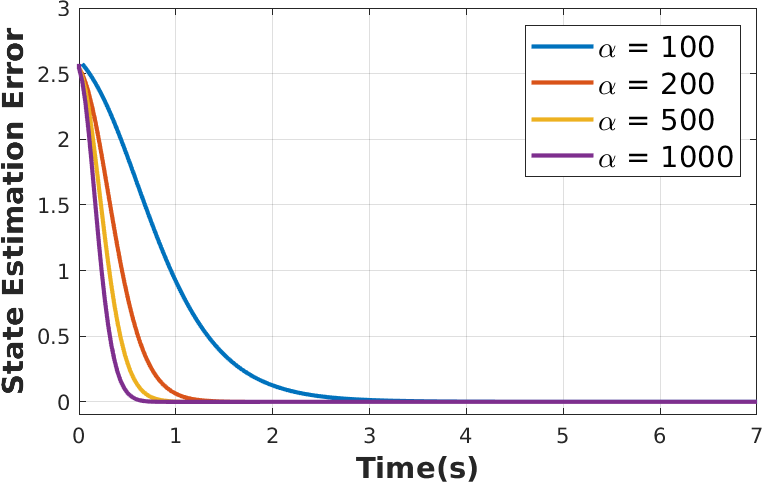}
        \label{fig:spicadiffgain}
    }
    \hfill
    \subfloat[Convergence of the \emph{MHO+M-P} method for increasing window sizes.]{
        \includegraphics[width=0.3\textwidth]{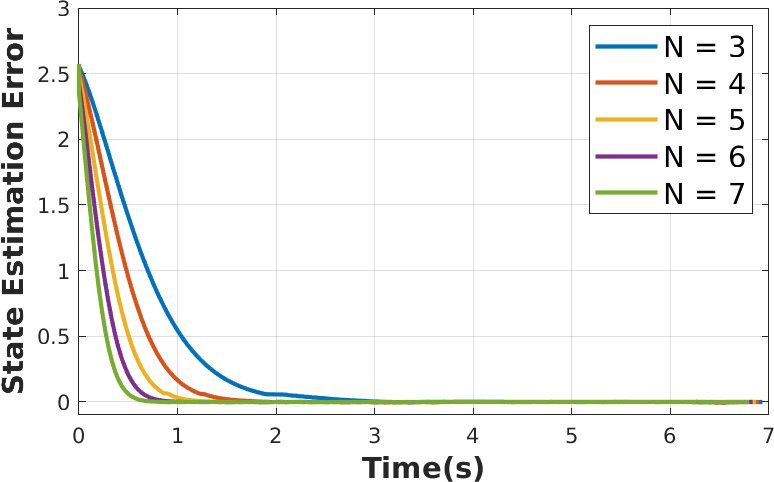}
        \label{fig:mhodiffN}
    }
    \hfill
    \subfloat[Convergence of the \emph{MHO+M-P} method for increasing gains.]{
        \includegraphics[width=0.3\textwidth]{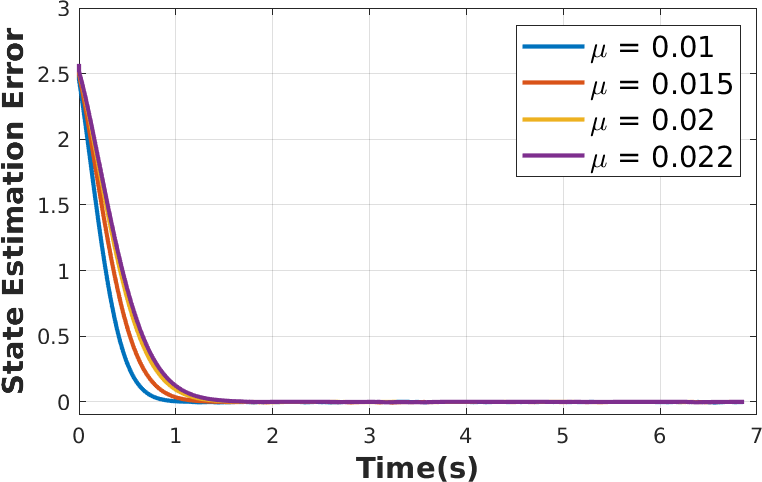}
        \label{fig:mhodiffMu}
    }
    \caption{\it Evaluation of the convergence time of the MLO and the MHO for different parameter values. For the former different gains were considered, while for the latter, both the gain and window size were varied.
    The variable being presented is one of the entries of vector $\bm{\chi}$ of the \emph{M-P} representation.}
    \label{fig:simNoiselessResults}
\end{figure*}

{\setlength{\tabcolsep}{2pt}
\begin{table}[t]
    \vspace{.1cm}
    \centering
    \subfloat[Convergence Time for the \emph{Sphere} representation with both \emph{MLO} and \emph{MHO} observer.]{
        \resizebox{.5\textwidth}{!}{
        \begin{tabular}{|c|c|c|c|c|c|c|c|c|c|c|c|c|c|}
            \cline{2-14}
            \multicolumn{1}{c|}{} &
            \multicolumn{4}{c|}{\bf \emph{MLO+Sphere} \cite{mateus2019}} &
            \multicolumn{9}{c|}{\bf \emph{MHO+Sphere}} \\
            \cline{2-14}
            \multicolumn{1}{c|}{} & 
            \multicolumn{4}{c|}{$\alpha$} &
            \multicolumn{5}{c|}{$N$, $\mu = 0.014$} &
            \multicolumn{4}{c|}{$\mu$, $N = 5$} \\
            \cline{2-14}
            \multicolumn{1}{c|}{} & 100 & 200 & 500 & 1000 & 3 & 4 & 5 & 6 & 7 & 0.01 & 0.015 & 0.02 & 0.022 \\
            \hline
            \makecell{Mean Conv. \\ Time (sec)} & 2.088 & 1.606 & 1.137 & 0.880 & 3.511 & 2.127 & 1.542 & 1.220 & 0.880 & 1.242 & 1.559 & 1.892 & 2.212 \\
            \hline
            \makecell{Median Conv. \\ Time (sec)} & 1.967 & 1.500 & 1.000 & 0.767 & 3.300 & 1.933 & 1.300 & 0.967 & 0.733 & 1.033 & 1.417 & 1.617 & 1.767 \\
            \hline
            \end{tabular}
        }
    }\\
    \subfloat[Convergence Time for the \emph{M-P} representation with both \emph{MLO} and \emph{MHO} observer.]{
        \resizebox{.5\textwidth}{!}{
        \begin{tabular}{|c|c|c|c|c|c|c|c|c|c|c|c|c|c|}
            \cline{2-14}
            \multicolumn{1}{c|}{} &
            \multicolumn{4}{c|}{\bf \emph{MLO+M-P}} &
            \multicolumn{9}{c|}{\bf \emph{MHO+M-P}} \\
            \cline{2-14}
            \multicolumn{1}{c|}{} & 
            \multicolumn{4}{c|}{$\alpha$} &
            \multicolumn{5}{c|}{$N$, $\mu = 0.014$} &
            \multicolumn{4}{c|}{$\mu$, $N = 5$} \\
            \cline{2-14}
            \multicolumn{1}{c|}{} & 100 & 200 & 500 & 1000 & 3 & 4 & 5 & 6 & 7 & 0.01 & 0.015 & 0.02 & 0.022 \\
            \hline
            \makecell{Mean Conv. \\ Time (sec)} & 2.666 & 2.186 & 1.671 & 1.349 & 4.580 & 3.251 & 2.592 & 1.799 & 1.587 & 2.095 & 2.383 & 3.461 & 3.148 \\
            \hline
            \makecell{Median Conv. \\ Time (sec)} & 2.267 & 1.667 & 1.133 & 0.833 & 4.550 & 2.167 & 1.733 & 1.200 & 0.833 & 1.433 & 1.767 & 2.300 & 2.233 \\
            \hline
            \end{tabular}
        }
    }
    
    \caption{\it Mean and median convergence times of observers \emph{MLO} in \cite{mateus2019}, and \emph{MHO} in \cite{alessandri2008}, for the estimation of lines with the \emph{Sphere} and \emph{M-P} representations (see Sec.~\ref{sec:spherical} and Sec.~\ref{sec:newlinerep} respectively). 
    }
    \label{tab:noiseFree}
\end{table}
}

\vspace{.15cm}
\noindent
{\bf Noisy data:} 
This simulation test consists of evaluating the observers with increasing levels of noise.
Noise is added to the measurements, in this case, the vector $\mathbf{m}$, which can be retrieved directly from the image.
Since it is a unit vector, perturbations are added by sampling a random rotation matrix and multiplying it to the moment vector.
The three Euler angles are sampled from a uniform distribution with increasing standard deviation to generate this matrix, and then the angles are converted to a rotation matrix.
Notice that $\theta$ and $\phi$ of the \emph{Sphere} representation are computed from $\mathbf{m}$. Thus noise is only added in the vector and not to the angles.
The direction and depth errors are defined as
\begin{equation}
    \epsilon_{\mathbf{d}} = \arccos(\hat{\mathbf{d}}_i^T\mathbf{d}_i), \ \text{and} \ \
    \epsilon_l =  \| \hat{l}_i - l_i \|,
    \label{eq:errors}
\end{equation}
respectively.
The gain $\alpha$ was set to $1000$, and the window size and gain of the \emph{MHO} were set as $N = 7$ and $\mu = 0.014$ -- for both representations --, respectively. 
The median direction and depth errors for $100$ runs of each noise level are presented in Figs.~\ref{fig:simNoisyResults}\subref{fig:derror}, and \ref{fig:simNoisyResults}\subref{fig:lerror}, respectively.
The \emph{MHO+M-P}, and \emph{MHO+Sphere} present a significant improvement in performance when compared to \emph{MLO+Sphere}, and \emph{MLO+M-P}, with the \emph{M-P} performing slightly better. 
This behavior is expected because it considers measurement and model noise, while the continuous observers do not.
Between the observers in \eqref{eq:spherical_observer} and \eqref{eq:newobserver}, the one using the \emph{M-P} line formulation in Sec.~\ref{sec:newlinerep} -- which is the new model presented in this paper -- performs better.
This is due to the fewer number of steps to recover the \emph{Pl\"ucker coordinates} than the \emph{Sphere} representation.
Recall, that \emph{M-P} only requires a cross product (see \eqref{eq:chi}), while the parameterization \emph{Sphere} requires to retrieve $\mathbf{m}$ from the spherical angles, and then $\mathbf{d}$ by inverting the orthonormal basis in Sec.~\ref{sec:newlinerep}.

\begin{figure}
    \centering
    \subfloat[Median line direction error.]{
        \includegraphics[width=0.23\textwidth]{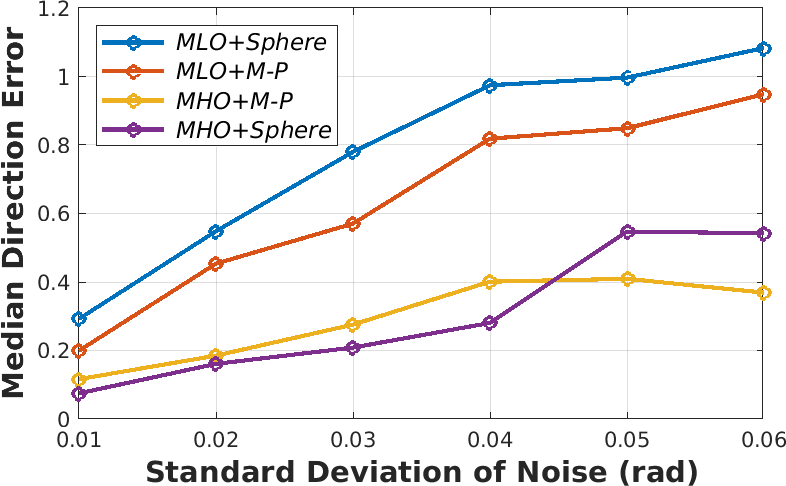}
        \label{fig:derror}
    }\hfill
    \subfloat[Median line depth error.]{
        \includegraphics[width=0.23\textwidth]{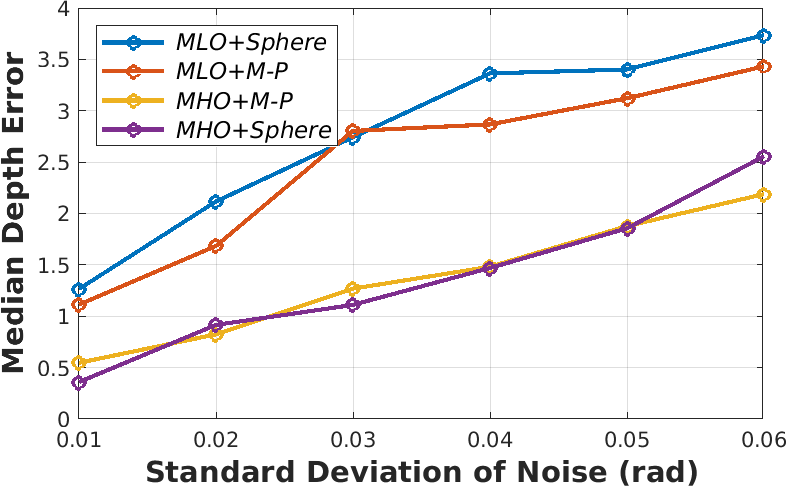}
        \label{fig:lerror}
    }
    \caption{\it Evaluation of the line direction and depth errors for increasing levels of measurement noise.}
    \label{fig:simNoisyResults}
\end{figure}

\vspace{.15cm}
\noindent
{\bf Non observable time intervals:}
In this simulation tests, the setup (of the camera and line) is similar to the noise-free tests. However, we consider time periods where the camera's linear velocity is such that the system is not observable.
This happens when the linear velocity is null or it is contained in the line interpretation plane, i.e., $\bm{\nu}_c^T\mathbf{m} = 0$.
The tests consist of estimating a line for $10$ seconds and alternating the camera velocity so that the system is observable for 1 second and non-observable for another.
Four combinations of possible linear and angular velocities were considered: 1) the angular velocity is null, and the linear velocity alternates between null and non-null values; 2) The linear velocity behaves as in 1), but the angular velocity is not null; 3) The angular velocity is null, and the linear velocity alternates between being or not in the line interpretation plane; 4) The linear velocity behaves as in 3), but the angular velocity is not null.
These scenarios are not arbitrary since they have practical significance.
Scenarios 1) and 2) are common in differential drive robots when they move forward and stop to rotate about the kinematic center.
Scenarios 3) and 4) can happen when a mobile robot moves in the floor plane and attempts to estimate horizontal lines.
The magnitude of the estimation error, i.e., norm of the estimation error of the \emph{MLO}, and \emph{MHO} is presented in Fig.~\ref{fig:nonObsTimeInt}.
We can see that when the camera comes to a complete stop (see Fig.~\ref{fig:nonObsTimeInt}~\subref{fig:magv0vw0}), the norm of the error does not increase.
The same behavior is verified for 2), Fig.~\ref{fig:nonObsTimeInt}~\subref{fig:magv0vw}, since the movement for null linear velocity consist of a rotation, which does not affect the depth of the line.
In Fig.~\ref{fig:nonObsTimeInt}~\subref{fig:magvtm0vw0} and Fig.~\ref{fig:nonObsTimeInt}~\subref{fig:magvtm0vw}, we can see that the error increases in the non observable regions, and it increases more the higher the current norm of the error.
Nonetheless, since the non-observable time intervals are not arbitrarily large, both observers can converge.
When the norm of the error is close to zero, the divergence is not significant.

\begin{figure*}
    \centering
    \subfloat[Norm of the error for $\bm{\omega}_c = 0$, and the linear velocity alternating between non-null and zero.]{
        \includegraphics[width=0.23\textwidth]{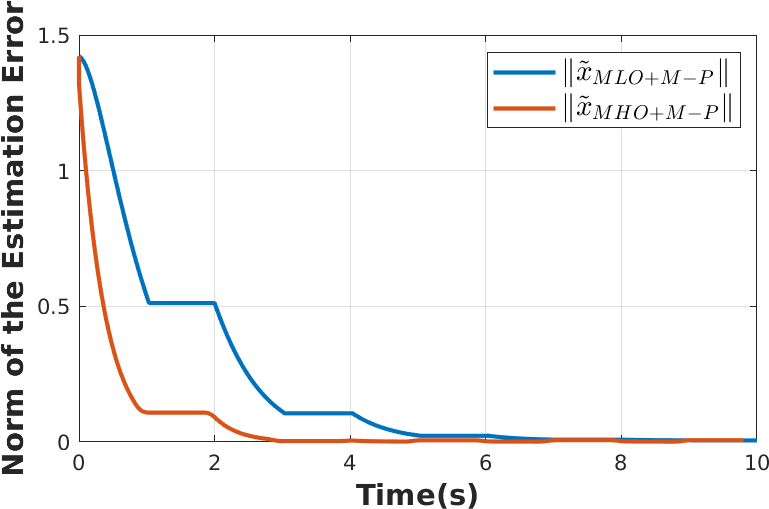}
        \label{fig:magv0vw0}
    }
    \hfill
    \subfloat[Norm of the error for $\bm{\omega}_c \neq 0$, and the linear velocity alternating between non-null and zero.]{
        \includegraphics[width=0.23\textwidth]{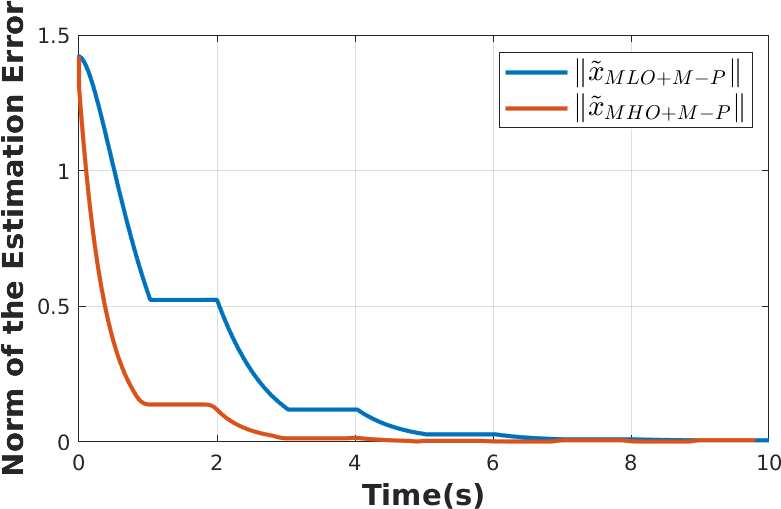}
        \label{fig:magv0vw}
    }
    \hfill
    \subfloat[Norm of the error for $\bm{\omega}_c = 0$, and the linear velocity alternating between $\bm{\nu}_c^T\mathbf{m} \neq 0$, and $\bm{\nu}_c^T\mathbf{m} = 0$.]{
        \includegraphics[width=0.23\textwidth]{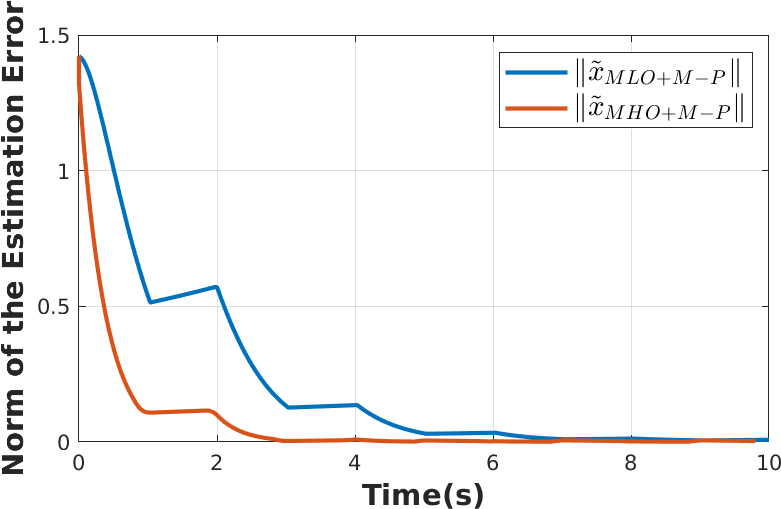}
        \label{fig:magvtm0vw0}
    }
    \hfill
    \subfloat[Norm of the error for $\bm{\omega}_c \neq 0$, and the linear velocity alternating between $\bm{\nu}_c^T\mathbf{m} \neq 0$, and $\bm{\nu}_c^T\mathbf{m} = 0$.]{
        \includegraphics[width=0.235\textwidth]{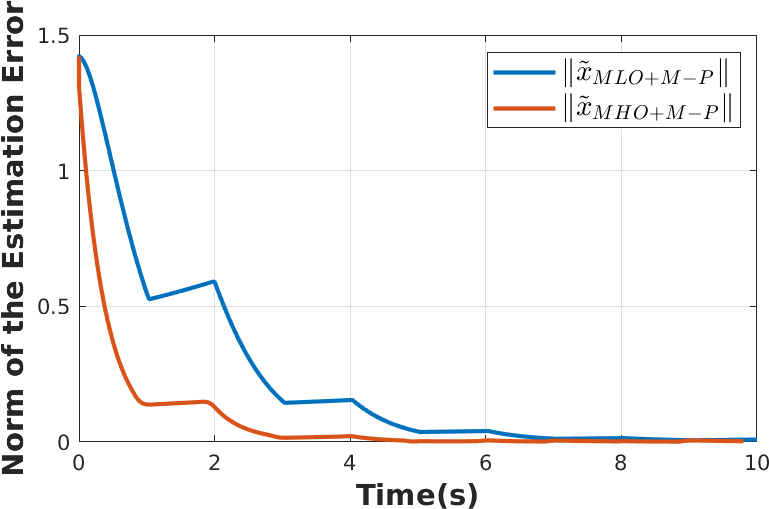}
        \label{fig:magvtm0vw}
    }
    \caption{\it Magnitude of the state estimation error of the \emph{MLO+M-P}, and \emph{MHO+M-P} for different combinations of the linear and angular velocity. On the left, the camera moves with null angular velocity, and the linear velocity alternates between null e non-null values. On the middle left, the camera moves with the same linear velocity as on the left, but the angular velocity is non-null. On the middle right, the camera moves with null angular velocity, and the linear velocity alternates between being or not contained in the line interpretation plane. Finally, on the right, the camera moves as in the previous experiment, but with non-null angular velocity.}
    \label{fig:nonObsTimeInt}
\end{figure*}

\subsection{Real Data Results}
\label{sec:realResults}

This section presents the results of the \emph{MHO+M-P} compared with the \emph{MLO+M-P}.
The experimental setup consists of a camera mounted on two robots, an omnidirectional mobile robot MBOT \cite{messias2014}, and a $7$ degrees-of-freedom manipulator Baxter Research Robot.
The robots' APIs are available in the Robot Operating System (ROS) \cite{quigley2009}.
For feature/line tracking, the Visual Servoing Platform (ViSP) \cite{marchand2005} was used, namely the moving-edge tracker \cite{marchand2005b}.
Ground-truth data was acquired by computing the pose of the camera concerning the target where the line lies.
The pose is computed by the POSIT algorithm in \cite{oberkampf1996}, from the detection of four known points in the target reference frame.
The code for data collection is implemented in \emph{C++}.
For the observers, the same implementation as Sec.~\ref{sec:sim} was used.

 \begin{figure}
     \vspace{-.15cm}
     \centering
     \subfloat[MBOT robot setup.]{
         \includegraphics[height=0.125\textheight]{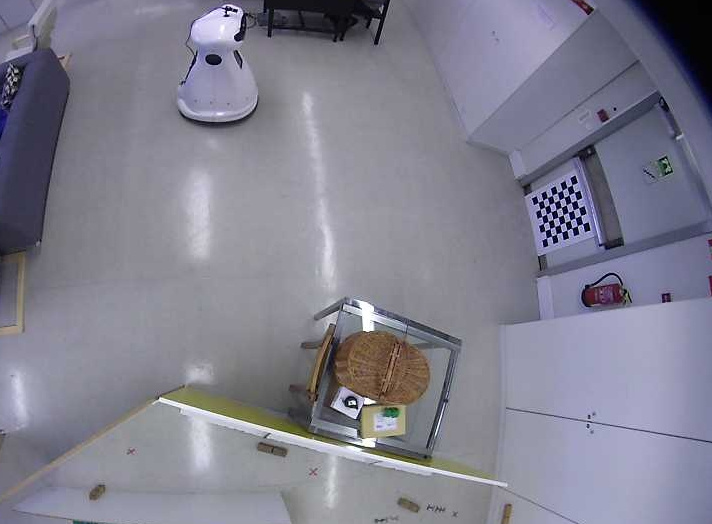}
         \label{fig:mbot}
     }\hfill
     \subfloat[Camera image with tracked line and points for pose estimation.]{
         \includegraphics[height=0.125\textheight]{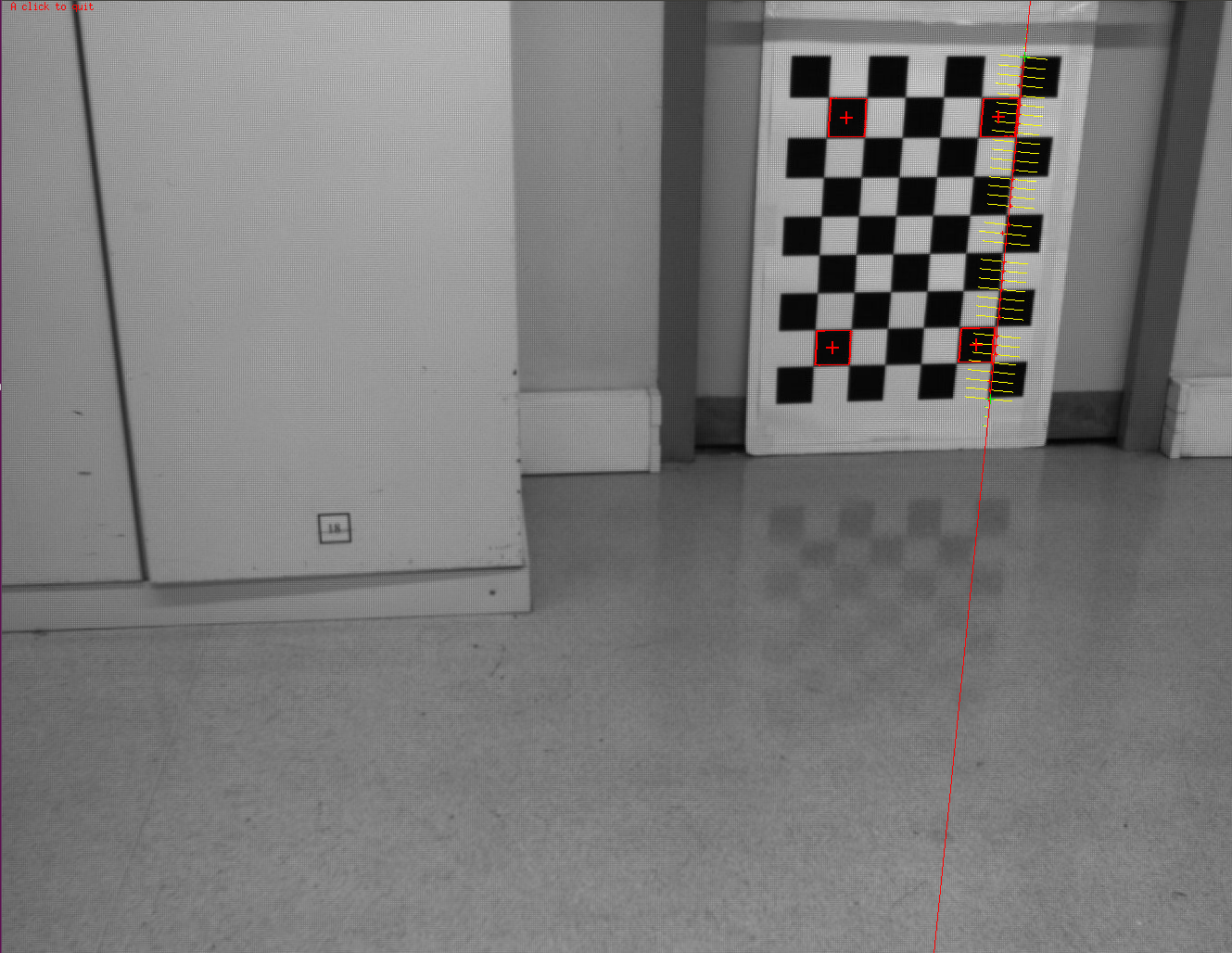}
         \label{fig:line_track}
     }
     \caption{\it On the left, the experimental setup, with the MBOT platform,  is shown. It is an omnidirectional platform, with 3 dof (2 linear and 1 angular). On top, there is a Pointgrey Flea3 USB3 camera. On the right, an image from the camera with the chessboard's four points, used for pose estimation (to obtain the true coordinates of the line) and the line tracked with the moving-edges tracker \cite{marchand2005b}.}
     \label{fig:mbotSetup}
 \end{figure}

\begin{figure}[!tbp]
  \centering
  \begin{minipage}[t]{0.235\textwidth}
    \includegraphics[height=.12\textheight]{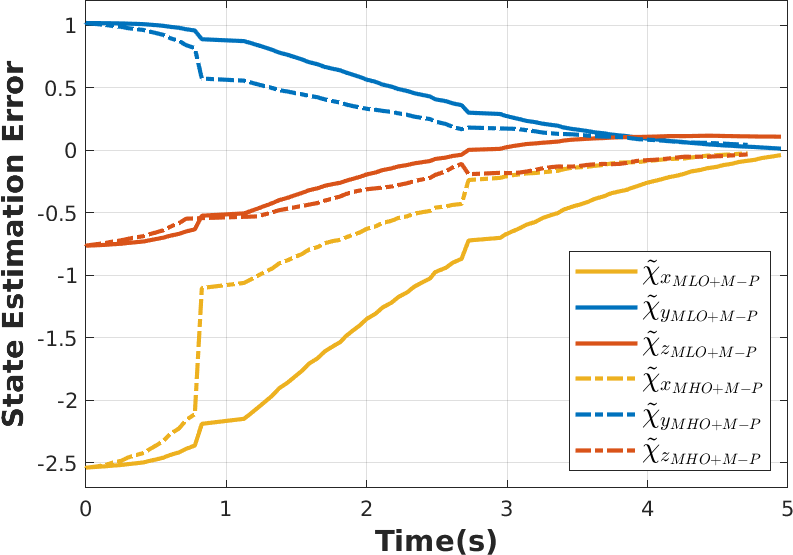}
    \caption{\it State estimation error of the \emph{MLO+M-P} and \emph{MHO+M-P} with real data acquired by the MBOT platform.}
    \label{fig:mhoRealMBOT}
  \end{minipage}
  \hfill
  \begin{minipage}[t]{0.235\textwidth}
    \includegraphics[height=.12\textheight]{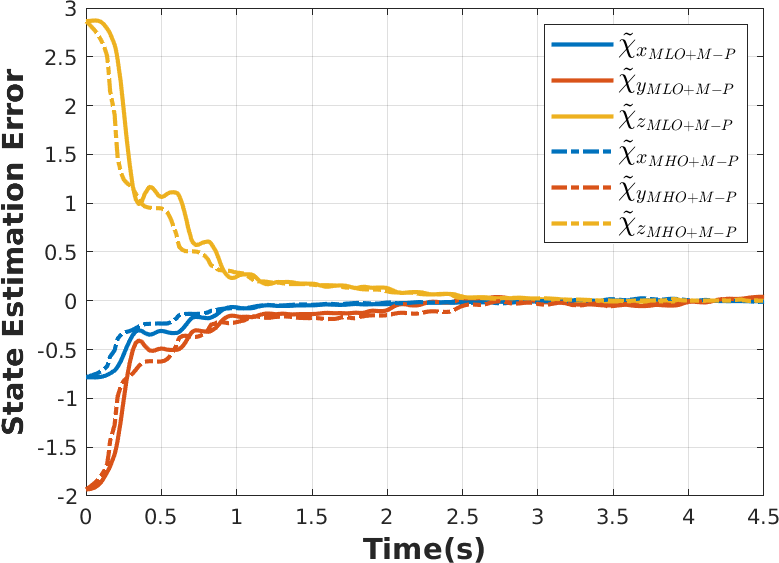}
    \caption{\it State estimation error with real data acquired by the Baxter robot.}
    \label{fig:mhoRealBaxter}
  \end{minipage}
\end{figure}

\vspace{.15cm}
\noindent
{\bf MBOT Results:}
In this set of experiments a Pointgrey Flea3 USB3 camera\footnote{https://www.ptgrey.com/flea3-usb3-vision-cameras}, which was mounted on the top of the robot, see Fig.~\ref{fig:mbotSetup}\subref{fig:mbot}, was used.
The target was a chessboard, and a line was defined on the leftmost edge of the board.
Four inner squares were chosen as the reference points for the pose estimation, see Fig.~\ref{fig:mbotSetup}\subref{fig:line_track}.
The gain $\alpha$ of the observer in \eqref{eq:newobserver} was set to $100$, and the window size and gain of the MHO were set as $N = 5$ and $\mu = 0.015$, respectively. 
For a fair comparison, these parameters were chosen such that the convergence time of both observers is similar and to have a small gain $\alpha$, so that the observer in \eqref{eq:newobserver} does not suffer much with noise amplification. 
A plot of the state estimation error for both observers is presented in Fig.~\ref{fig:mhoRealMBOT}.
For readability purposes, only the error of the unknown vector $\bm{\chi}$ is presented.
A total of six runs were performed with the MBOT platform. The results of the direction and depth errors as defined in \eqref{eq:errors} are presented in Tab.~\ref{tab:spicaRealMBOT}.
The line reconstruction with this platform results in poor results since the velocity readings are noisy. They are obtained from the odometer, i.e., by direct integration of the wheel motor encoders.
To reduce the effect of velocity readings noise, and since the robot velocity is near-constant through the runs, a low pass filter was applied to the signal.
We can see that the \emph{MHO+M-P} outperforms the \emph{MLO+M-P} in all runs except for the second.
These results are expected since the former takes into account noise in the modeling, while the latter takes the measurements and model at face value.

\begin{table}[t]
    \vspace{.1cm}
    \centering
    \subfloat[{\it Direction estimation error on the MBOT datasets in radians.}]{
    \resizebox{.48\textwidth}{!}{
    \begin{tabular}{|c|c|c|c|c|c|c|}
        \cline{2-7}
        \multicolumn{1}{c|}{} & Run 1 & Run 2 & Run 3 & Run 4 & Run 5 & Run 6 \\
        \hline
        \emph{MLO+M-P} & 0.264 & {\bf 0.447} & 0.330 & 0.161 & 0.239 & 0.391 \\
        \hline
        \emph{MHO+M-P} & {\bf 0.178} & 0.460 & {\bf 0.105} & {\bf 0.090} & {\bf 0.123} & {\bf 0.179} \\
        \hline
    \end{tabular}
    }
    }\\
    \subfloat[{\it Depth estimation error on the MBOT datasets in meters.}]{
    \resizebox{.48\textwidth}{!}{
    \begin{tabular}{|c|c|c|c|c|c|c|}
        \cline{2-7}
        \multicolumn{1}{c|}{} & Run 1 & Run 2 & Run 3 & Run 4 & Run 5 & Run 6 \\
        \hline
        \emph{MLO+M-P} & 0.210 & 0.359 & 0.378 & 0.327 & 0.014 & 0.177  \\
        \hline
        \emph{MHO+M-P} & {\bf 0.068} & {\bf 0.289} & {\bf 0.181} & {\bf 0.111} & {\bf 0.093} & {\bf 0.085}  \\
        \hline
    \end{tabular}
    }
    }
    \caption{\it Direction and depth errors as defined in \eqref{eq:errors} for six sequences retrieved by the MBOT robot.}
    \label{tab:spicaRealMBOT}
\end{table}

\vspace{.15cm}
\noindent
{\bf Baxter Results:}
In this set of experiments, the camera on the left end-effector of a Baxter robot was used as the imaging device. 
The manipulator has seven degrees-of-freedom, and the pose of the camera concerning the end-effector is available; i.e., the camera is calibrated extrinsically.
The target is composed of six lines in two wooden boards placed in front of the camera, see Fig.~\ref{fig:baxterSetup}.
A single line is reconstructed at a time.
However, the vertical lines were not estimated.
This was because the camera's linear motion (vertical lift) was nearly co-planar with the lines' interpretation plane.
This is a singularity of the line dynamics, see \eqref{eq:mdyn}, which leads to a loss of observability in the sense of \cite{spica2013}.
Furthermore, to prevent self-collisions, the velocity of the manipulator was reduced.
This lead to the need to increase the gain $\alpha$ and the observation window size $N$, with respect to the MBOT sequences.
Recall that the estimation will be faster, the higher the norm of the linear velocity and/or the higher the gains.
Thus, $\alpha$ was set to $1000$, and the window size and gain of the \emph{MHO+M-P} were set as $N = 7$, and $\mu = 0.013$, respectively.
A plot of the state estimation error for one of the lines is presented in Fig.~\ref{fig:mhoRealBaxter}.
Notice that, given the high value of $\alpha$, the estimation curves of the \emph{MLO+M-P} are more susceptible to noise, and the curves \emph{MHO+M-P} are smoother.
The direction and depth errors as defined in \eqref{eq:errors} are presented in Tab.~\ref{tab:spicaRealBaxter}.
Given the lower noise in the manipulator's velocity readings with respect to the MBOT, the errors are smaller than those obtained by the mobile robot.
Besides, the lower speed (norm of the linear velocity), and more stable camera position, results in lower measurement noise, leading to better results.
Even though this data presents less noise than the MBOT data, the \emph{MHO+M-P} continues to outperform the \emph{MLO+M-P} method.

\begin{figure}
    \vspace{-.25cm}
    \centering
    \includegraphics[height=.2\textheight]{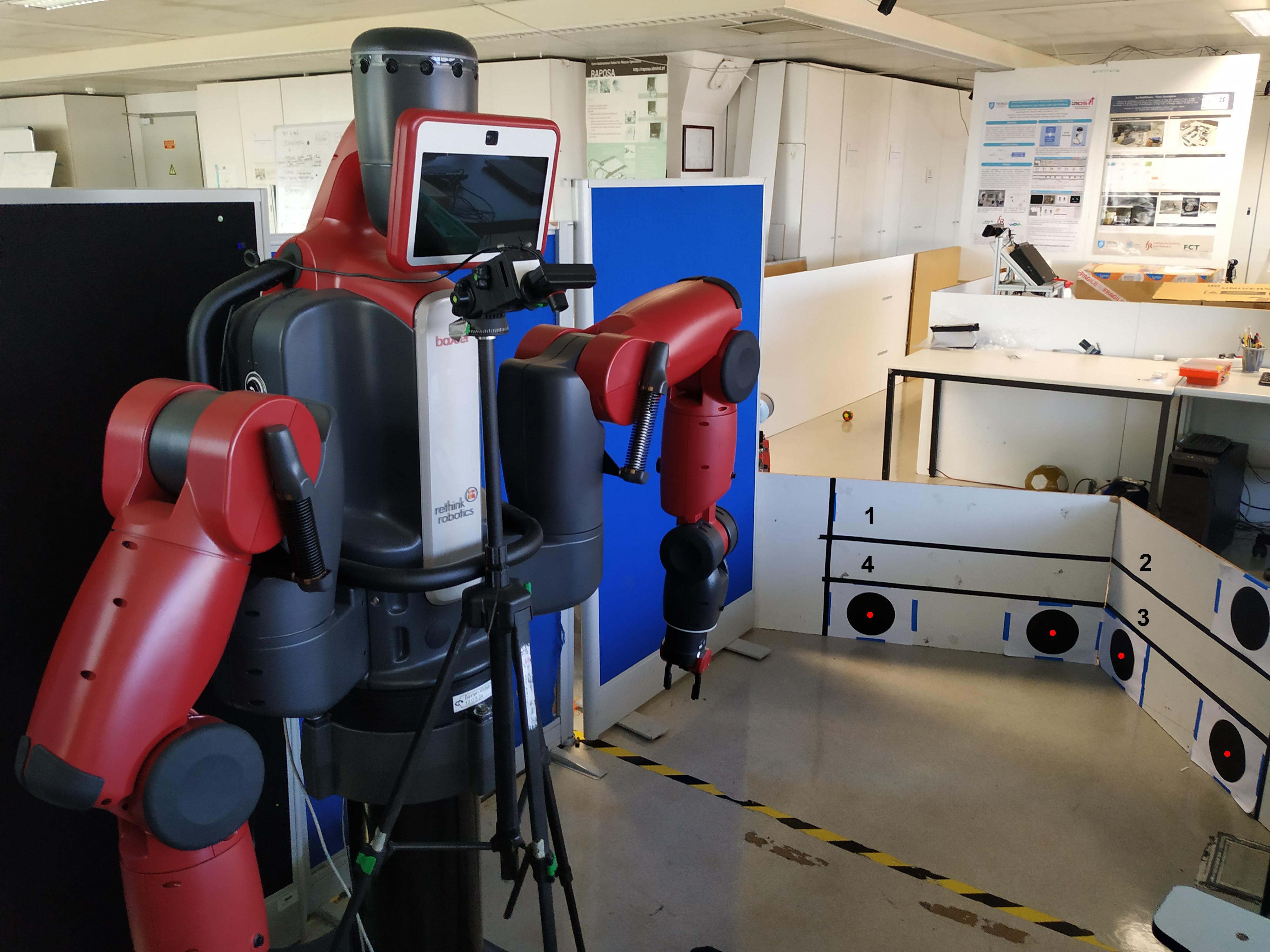}
    \caption{\it Experimental setup of the Baxter robot, with two seven degrees-of-freedom manipulators, equipped with cameras in their end-effectors. The target is composed of two wooden boards, with the numbers corresponding to the ones in Tab.~\ref{tab:spicaRealBaxter}. Furthermore, the four points used to retrieve the target's pose with respect to the camera are also depicted.
    }
    \label{fig:baxterSetup}
\end{figure}

\begin{table}[t]
    \vspace{.1cm}
    \centering
    \subfloat[{\it Direction estimation error on the Baxter dataset in radians.}]{
    \resizebox{.45\textwidth}{!}{
    \begin{tabular}{|c|c|c|c|c|}
        \cline{2-5}
        \multicolumn{1}{c|}{} & Line 1 & Line 2 & Line 3 & Line 4\\
        \hline
        \emph{MLO+M-P} & {\bf 0.062} & 0.061 & 0.059 & 0.60 \\
        \hline
        \emph{MHO+M-P} & 0.067 & {\bf 0.016} & {\bf 0.014} & {\bf 0.050} \\
        \hline
    \end{tabular}
    }
    }\\
    \subfloat[{\it Depth estimation error on the Baxter dataset in meters.}]{
    \resizebox{.45\textwidth}{!}{
    \begin{tabular}{|c|c|c|c|c|}
        \cline{2-5}
        \multicolumn{1}{c|}{} & Line 1 & Line 2 & Line 3 & Line 4 \\
        \hline
        \emph{MLO+M-P} & 0.006 & 0.063 & 0.077 & 0.003 \\
        \hline
        \emph{MHO+M-P} & {\bf 0.002} & {\bf 0.020} & {\bf 0.027} & {\bf 0.001} \\
        \hline
    \end{tabular}
    }
    }
    
    \caption{\it Direction and depth errors computed by \eqref{eq:errors} for the four lines depicted in Fig.~\ref{fig:baxterSetup} using the Baxter robot.}
    \label{tab:spicaRealBaxter}
\end{table}
\section{Conclusions}
\label{sec:conclusion}

In this work, we presented and evaluated different methods to perform incremental SfM using line features.
Given that the line dynamics in Pl\"ucker coordinates are not fully observable, we present two alternative representations.
The \emph{Sphere} representation was proposed in our previous work. It consisted of exploiting the fact that the measurement is a unit vector representing it in spherical coordinates to obtain a minimal line parameterization.
However, this formulation presents singularities.
To tackle this problem, we present a novel representation \emph{M-P} that models lines based on the moment vector and the view ray of the closest point of the line to the camera's optical center.
Structure-from-motion is achieved by exploiting state observers.
Two different methodologies were presented.
The \emph{MLO}, which is a state-of-the-art framework to estimate visual dynamical systems, and we show how to design the observer for both presented line representations.
The second estimation approach consists of a \emph{MHO}.
To the best of our knowledge, it is the first time these observers are applied to the SfM task.
The stability of the \emph{MHO} is analyzed, and it is shown that the estimation error is bounded for the \emph{M-P} representation.
Extensive simulation and real experiments were conducted, showing that the \emph{MHO} is more robust to measurement noise while achieving similar convergence times.
Concerning the two representations, the \emph{M-P} does not have singularities. Even though the observer for its system is slower than the \emph{Sphere} representation, it is more robust to measurement noise.

Future work includes coupling the \emph{MHO+M-P} in a Visual Servoing control.
The most straightforward approach is to plug the \emph{MHO+M-P} in the Visual Servoing task as in \cite{deluca2008}.
Another possibility is --similar to \cite{spica2017,rodrigues2020}-- to develop coupling mechanisms between standard Visual Servoing and the \emph{MHO+M-P}, to obtain control, which not only leads the features to the reference value but also optimize the path to improve the depth estimation.
Still, in the scope of Visual Servoing, we aim at exploiting the coupling of Model Predictive Control with our Moving Horizon Observer as in \cite{copp2014}.
One limitation of our method is the computational cost of estimating multiple lines.
Thus, we aim to exploit alternative solvers for our optimization problem, which would allow us to analyze the complexity of the entire estimation process. This analysis is not available for the solver used in this work (\cite{lagarias1998}).

\section{Acknowledgements}

We would like to thank the editor and the reviewers for the time devoted to reviewing our paper and the comments provided. This work was partially supported by the Portuguese FCT grants {\tt PD/BD/135015/2017} (through the NETSys Doctoral Program), LARSyS - FCT Project {\tt UIDB/50009/2020}, the RBCog-Lab research infrastructure {\tt PINFRA/22084/2016}, and by the European Union from the European Regional Development Fund under the Smart Growth Operational Programme as part of the project {\tt  POIR.01.01.01-00-0102/20}, with title "Development of an innovative, autonomous mobile robot for retail stores".

\bibliographystyle{IEEEtran}
\bibliography{IEEEabrv,references}

\vspace{-.75cm}
\begin{IEEEbiography}[{\includegraphics[width=1in,height=1.25in,clip,keepaspectratio]{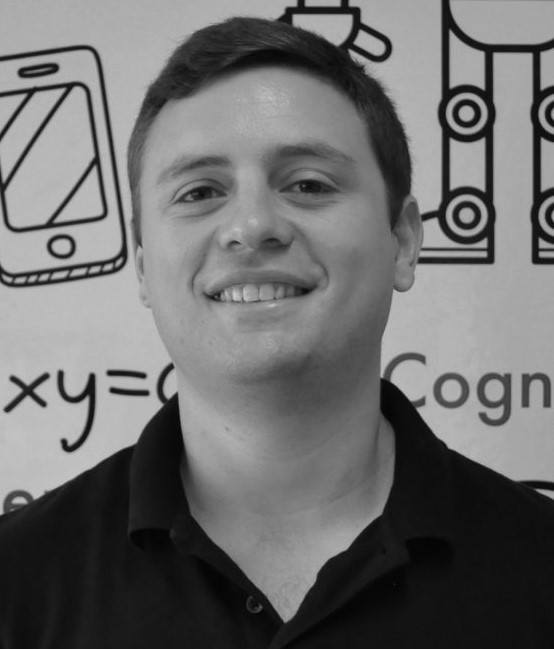}}]%
{Andr\'e Mateus}
was born in Nazar\'e, Portugal in 1991. He received his M. Sc. in Electrical and Computer Engineering from University of Lisbon in 2015. Currently, he is a fourth year Ph. D. student at Instuto Superior T\'ecnico, University of Lisbon, and  a Reseach Assistant at the Institute for Systems and Robotics, Lisbon. His research interests include Computer Vision and Robotics. He is a student member of IEEE.
\end{IEEEbiography}

\vspace{-.75cm}
\begin{IEEEbiography}[{\includegraphics[width=1in,height=1.25in,clip,keepaspectratio]{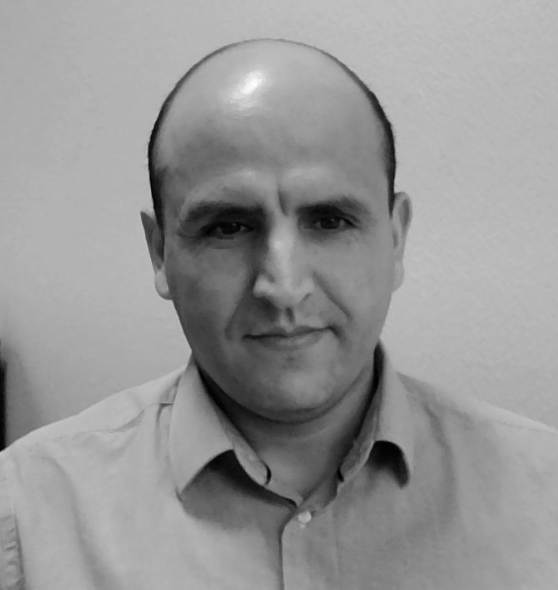}}]%
{Omar Tahri}
was born in Fez, Morocco, in 1976. He got his Master in photonics, images and system control from the Louis Pasteur University, Strasbourg, France, in 2000 and received his Ph.D.
degree in computer science from the University of Rennes, France, in March 2004. His research interests include robotics and computer vision. He is Full Professor and the head of the robot vision research team ImViA-ViBot at the university of Burgundy.
\end{IEEEbiography}

\vspace{-.75cm}
\begin{IEEEbiography}[{\includegraphics[width=1in,height=1.25in,clip,keepaspectratio]{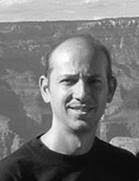}}]{A. Pedro Aguiar} (Senior Member, IEEE) received the Licenciatura, M.S., and Ph.D. degrees in Electrical and Computer Engineering (ECE) from the Instituto Superior T\'ecnico (IST), University of Lisbon, Portugal in 1994, 1998, and 2002, respectively. He is currently an Associate Professor with the ECE Department, Faculty of Engineering, University of Porto, Portugal. From 2002 to 2005, he was a Postdoctoral Researcher at the Center for Control, Dynamical-Systems, and Computation, University of California at Santa Barbara, CA, USA. From 2005 to 2012, he was an invited Professor with the IST. His research interests include motion control and perception systems of single and multiple autonomous robotic vehicles, and nonlinear control theory and applications.
\end{IEEEbiography}

\vspace{-.75cm}
\begin{IEEEbiography}[{\includegraphics[width=1in,height=1.25in,clip,keepaspectratio]{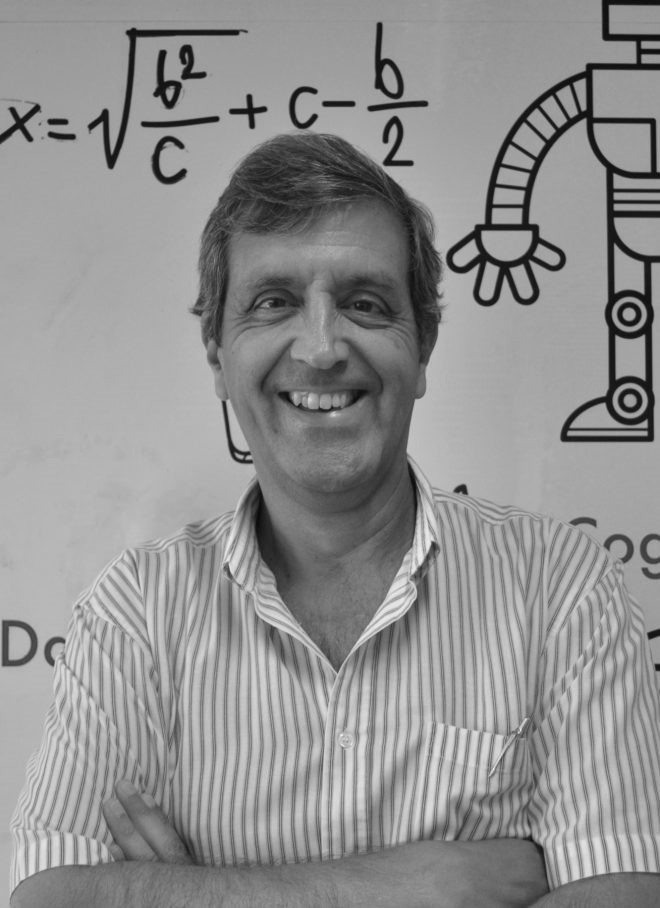}}]
{Pedro U. Lima} received the Ph.D. (1994) in Electrical Engineering from the Rensselaer Polytechnic Institute, NY, USA. Currently, he is a Full Professor at IST, Universidade de Lisboa, where he serves as the Deputy President for Scientific Matters of the Institute for Systems and Robotics.
Pedro was President and founding member of the Portuguese Robotics Society, and was awarded a 6-month Chair of Excellence at the Universidad Carlos III de Madrid, Spain in 2010. He is a Trustee of the RoboCup Federation.
His research interests lie in the areas of discrete event models of robot tasks and planning under uncertainty, with applications to networked robot systems.
\end{IEEEbiography}

\vspace{-.75cm}
\begin{IEEEbiography}[{\includegraphics[width=1in,height=1.25in,clip,keepaspectratio]{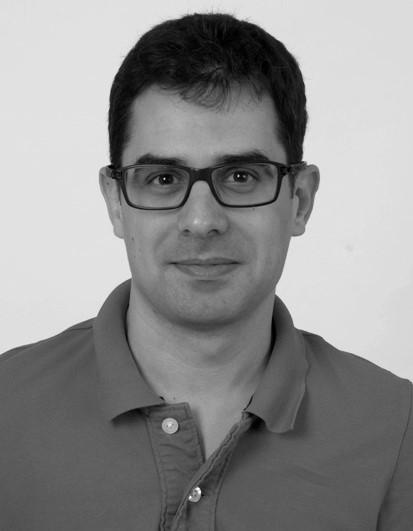}}]%
{Pedro Miraldo} received the Master and Doctoral degrees in Electrical and Computer Engineering from the University of Coimbra, Portugal, in 2008 and 2013. He was a postdoctoral researcher at Instituto Superior T\'ecncico (IST), University of Lisboa, Portugal, and KTH Royal Institute of Technology, Stockholm, Sweden, from 2014 to 2019. Currently, he is an Assistant Research at IST,
with more than 30 peer-review papers in the top conferences and journal articles. His research interests include 3D computer vision and robotics. He is a member of IEEE.
\end{IEEEbiography}

\end{document}